\documentclass[twoside,11pt]{article}

\usepackage{jmlr2e_arxiv}

\usepackage{enumitem}
\usepackage{color}

\newcommand{\honda}[1]{}
\newcommand{\komiyama}[1]{}
\newcommand{\myred}[1]{#1}
\newcommand{\updated}[1]{#1}

\newenvironment{proofof}[1]{\par\noindent{\bf #1\ }}{\hfill\BlackBox\\[2mm]}

\jmlrheading{1}{2023}{1-48}{4/00}{10/00}{meila00a}{Junpei Komiyama, Edouard Fouch\'e, and Junya Honda}

\ShortHeadings{Globally Nonstationary Bandits}{Komiyama, Fouch\'e, and Honda}
\firstpageno{1}

\usepackage{algorithm,algpseudocode}
\usepackage[utf8]{inputenc}
\usepackage{subfigure}
\usepackage{bm}
\usepackage{comment}

\newtheorem{thm}{Theorem}
\newtheorem{lem}[thm]{Lemma}

\newtheorem{mydef}[thm]{Definition}
\newtheorem{myrem}[thm]{Remark}
\newtheorem{assumption}[thm]{Assumption}
\newenvironment{sketch}{\par\noindent{\bf Proof Sketch\ }}{\hfill\BlackBox\\[2mm]}

\newcommand{\Natural}{\mathbb{N}}
\newcommand{\hatmu}{\hat{\mu}}
\newcommand{\tilmu}{\theta} 
\newcommand{\hatmunth}{\hat{\mu}^{(n)}}

\newcommand{\eps}{\epsilon}
\newcommand{\epscut}{\epsilon_{\mathrm{cut}}}
\newcommand{\epschg}[1]{\epsilon_{#1}}

\newcommand{\Nduring}[2]{N_{#1,#2}}

\newcommand{\Nit}{N_i(t)}
\newcommand{\Sit}{S_i(t)}

\newcommand{\EA}{\mathcal{A}}
\newcommand{\EB}{\mathcal{B}}
\newcommand{\EC}{\mathcal{C}}
\newcommand{\ED}{\mathcal{D}}

\newcommand{\EP}{\mathcal{P}}

\newcommand{\ES}{\mathcal{S}}
\newcommand{\ET}{\mathcal{T}}

\newcommand{\EV}{\mathcal{V}}
\newcommand{\EW}{\mathcal{W}}
\newcommand{\EX}{\mathcal{X}}
\newcommand{\EY}{\mathcal{Y}}

\newcommand{\tilO}{\tilde{O}}
\newcommand{\tilTheta}{\tilde{\Theta}}

\newcommand{\Err}{\mathrm{Err}}
\newcommand{\Reg}{\mathrm{Reg}}
\newcommand{\Regbase}{\mathrm{Reg}} 
\newcommand{\Regbasetor}{\mathrm{Reg}_{\mathrm{tr}}} 
\newcommand{\reg}{\mathrm{reg}} 

\newcommand{\Ex}{\mathbb{E}}

\newcommand{\Beta}{\mathrm{Beta}}
\newcommand{\KL}{d_{\mathrm{KL}}}
\newcommand{\Ind}{\mathbf{1}}
\newcommand{\nn}{\nonumber\\}

\newcommand{\Deltamin}{\Delta_{\mathrm{min}}}
\newcommand{\Td}[1]{T_{d,#1}}

\newcommand{\ADWIN}{ADS}
\newcommand{\AADWIN}{ADR}

\newcommand{\Cst}{C^{\mathrm{st}}}
\newcommand{\Cdr}{C^{\mathrm{dt}}} 

\newcommand{\Crabrupt}{C^{\mathrm{ab}}}
\newcommand{\Cdetect}{C^{\mathrm{gr}}}
\newcommand{\Mbreak}{M_d} 
\newcommand{\cbreak}{c^{\mathrm{br}}}
\usepackage{amsmath}

\newcommand{\NumChange}{M}

\newcommand{\CGrad}{C^{\mathrm{gr}}}

\newcommand{\niw}[2]{N_{#1,W_{#2}}}
\newcommand{\headt}{\underline{t}}
\newcommand{\tailt}{\overline{t}}
\newcommand{\since}[1]{\quad\left(\mbox{#1}\right)}
\newcommand{\n}{\nonumber}

\DeclareMathOperator*{\argmax}{arg\,max}

\newcommand{\Nbase}{N^{\mathrm{base}}}

\newcommand{\ETbase}{\ET^{\mathrm{base}}}
\newcommand{\ETmonitor}[1]{\ET^{\mathrm{monitor}}_{#1}}

\newcommand{\wasD}{K} 

\allowdisplaybreaks 

\begin{document}

\title{Finite-time Analysis of Globally \\ Nonstationary Multi-Armed Bandits} 

\author{\name Junpei Komiyama \email junpei@komiyama.info \\
       \addr Stern School of Business\\
       New York University\\
       44 West 4th Street, New York, NY, USA
       \AND
       \name Edouard Fouch\'e \email edouard.fouche@kit.edu \\
       \addr Institute for Program Structures and Data Organization\\
       Karlsruhe Institute of Technology\\
       Am Fasanengarten 5, 76131 Karlsruhe, Germany
       \AND
       \name Junya Honda \email honda@i.kyoto-u.ac.jp \\
       \addr Department of Systems Science, Graduate School of Informatics\\
       Kyoto University\\
       Yoshida Honmachi, Sakyo-ku, Kyoto, 606-8501, Japan}

\editor{(editor names)}

\maketitle

\begin{abstract}%
We consider nonstationary multi-armed bandit problems where the model parameters of the arms change over time. We introduce the adaptive resetting bandit (\AADWIN{}-bandit), a bandit algorithm class that leverages adaptive windowing techniques from literature on data streams.
We first provide new guarantees on the quality of estimators resulting from adaptive windowing techniques, which are of independent interest.
Furthermore, we conduct a finite-time analysis of \AADWIN{}-bandit in two typical environments: an abrupt environment where changes occur instantaneously and a gradual environment where changes occur progressively. We demonstrate that \AADWIN{}-bandit has nearly optimal performance when abrupt or gradual changes occur in a coordinated manner that we call global changes.
We demonstrate that forced exploration is unnecessary when we assume such global changes. Unlike the existing nonstationary bandit algorithms, \AADWIN{}-bandit has optimal performance in stationary environments as well as nonstationary environments with global changes.
Our experiments show that the proposed algorithms outperform the existing approaches in synthetic and real-world environments.
\end{abstract}


\section{Introduction}

\subsection{Motivations}
The multi-armed bandit (MAB; \cite{Thompson1933,robbins1952}) is a fundamental model capturing the dilemma between exploration and exploitation in sequential decision making. This problem involves $K$ arms (i.e., possible actions). The decision-maker selects a set of arms at each time step and observes a corresponding reward. The goal of the decision-maker is to maximize the cumulative reward over time. The performance of a bandit algorithm is usually measured via the metric of ``regret'': the difference between the obtained rewards and the rewards one would have obtained by choosing the best arms. Minimizing the regret corresponds to maximizing the expected reward. 

MAB has been used to solve numerous problems, such as experimental clinical design \citep{Thompson1933}, online recommendations \citep{li2010contextual}, online advertising \citep{DBLP:conf/nips/ChapelleL11,komiyama2015optimal}, and stream monitoring \citep{DBLP:conf/kdd/FoucheKB19}.
The most widely studied version of this model is the stochastic MAB, which assumes that the reward for each arm is drawn from an unknown but fixed distribution. In the stochastic MAB, several algorithms, such as the upper confidence bound (UCB; \cite{lairobbins1985,auer2002finite}) and Thompson sampling (TS; \cite{Thompson1933}) are known to have $\Theta(K \log T/\Deltamin)$ regret,\footnote{The value $\Deltamin$ is a distribution-dependent constant quantifying the hardness of the problem instance.} which is optimal \citep{lairobbins1985}.
While it is reasonable to assume in some cases that the reward-generating process does not change, as in these algorithms, the distribution of rewards may change over time in many applications. To observe this, we consider the following two examples:

\begin{example}{\rm (Online advertising)\label{expl:ad}
A website has several advertisement slots. Based on each user’s query, the website decides which ads to display from a set of candidates (i.e., ``relevant advertisements''). Some advertisements are more appealing to a user than others. Each advertisement is associated with a click-through rate (CTR), the number of clicks per view. Websites receive revenue from clicks on advertisements; thus, maximizing the CTR maximizes the revenue. This problem is structured as the bandit problem, with advertisements and clicks as arms and rewards. However, it is well known that the CTR of some advertisements may change over time for several reasons, such as seasonality or changing user interests. In this case, naively applying a stochastic MAB algorithm leads to suboptimal rewards.
}
\end{example}

\begin{example}{\rm (Predictive maintenance) \label{expl:bioliq}
Correlation often results from physical relationships, for example, between the temperature and pressure of a fluid. A change in these correlations often signals a shift in the system's state (like a fluid solidifying) or a failure or degradation in equipment (such as a leak). In the context of large-scale factory monitoring, keeping track of these correlations can help anticipate issues, thus reducing maintenance costs. Nevertheless, constantly updating the full correlation matrix is computationally unfeasible due to the data's high-dimensionality and dynamic nature. A more efficient solution consists of updating only a few elements of the matrix based on a notion of utility (e.g., high correlation values). The system must minimize the monitoring cost while maximizing the total utility in a possibly nonstationary environment. In other words, correlation monitoring can be considered an instance of the bandit problem \citep{DBLP:conf/kdd/FoucheKB19}, in which pairs of sensors and correlation coefficients correspond to arms and rewards.
}
\end{example}

In such settings, the reward may evolve over time (i.e., it is nonstationary\footnote{The use of the term ``nonstationary'' in the bandit literature is different from the literature on time series analysis. We formally define stationary and nonstationary streams in Definition \ref{def:streams}.}). The nonstationary MAB (NS-MAB)
describes a class of MAB algorithms addressing this particular setting. Most of the NS-MAB algorithms rely on passive forgetting methods based on a sliding window \citep{Garivier2008} or fixed-time resetting \citep{DBLP:conf/nips/GurZB14}. Recent work has proposed more sophisticated change detection mechanisms based on adaptive windows \citep{srivastava2014surveillance} or sequential likelihood ratio tests \citep{besson2020efficient}. However, the existing methods have several drawbacks, as we describe later. 
 
\subsection{Challenges in nonstationary bandits}

Although change detectors can help MAB algorithms adapt to changing rewards, they often come with costs. Let us consider the case of an abruptly changing environment, where the reward distributions change drastically at some time steps. 
Previous work by \cite{Garivier2008} indicated that the regret of any $K$-armed bandit algorithm for such a case is\footnote{
This holds even when the change is sufficiently large, regardless of the value of a distribution-dependent constant $\Deltamin$. See details in Theorem 13, Corollary 14, and Remark 17 in \cite{Garivier2008}.
} $\Omega(\sqrt{T})$. 
This finding implies that the performance of NS-MAB algorithms in stationary (i.e., non-changing) environments is inferior to the performance of standard stochastic MAB algorithms, such as the UCB and TS, because $O(\sqrt{T})$ is much larger than $O(K\log T/\Deltamin)$ given a moderate value of $\Deltamin$. Virtually all NS-MAB algorithms conduct $O(\sqrt{T})$ forced exploration for all the arms, which is the leading factor of regret---no matter what change-point detection algorithm is used.

Another drawback of most existing methods is that they require several parameters that are highly specific to the problem and require unreasonable environmental assumptions, such as the number of changes or an estimation of the amount of ``nonstationarity'' in the stream. If such parameters are not set correctly, the actual performance may deviate widely from the given theoretical bounds. 

\begin{figure}[t]
\begin{center}
    \includegraphics[width=1.\linewidth]{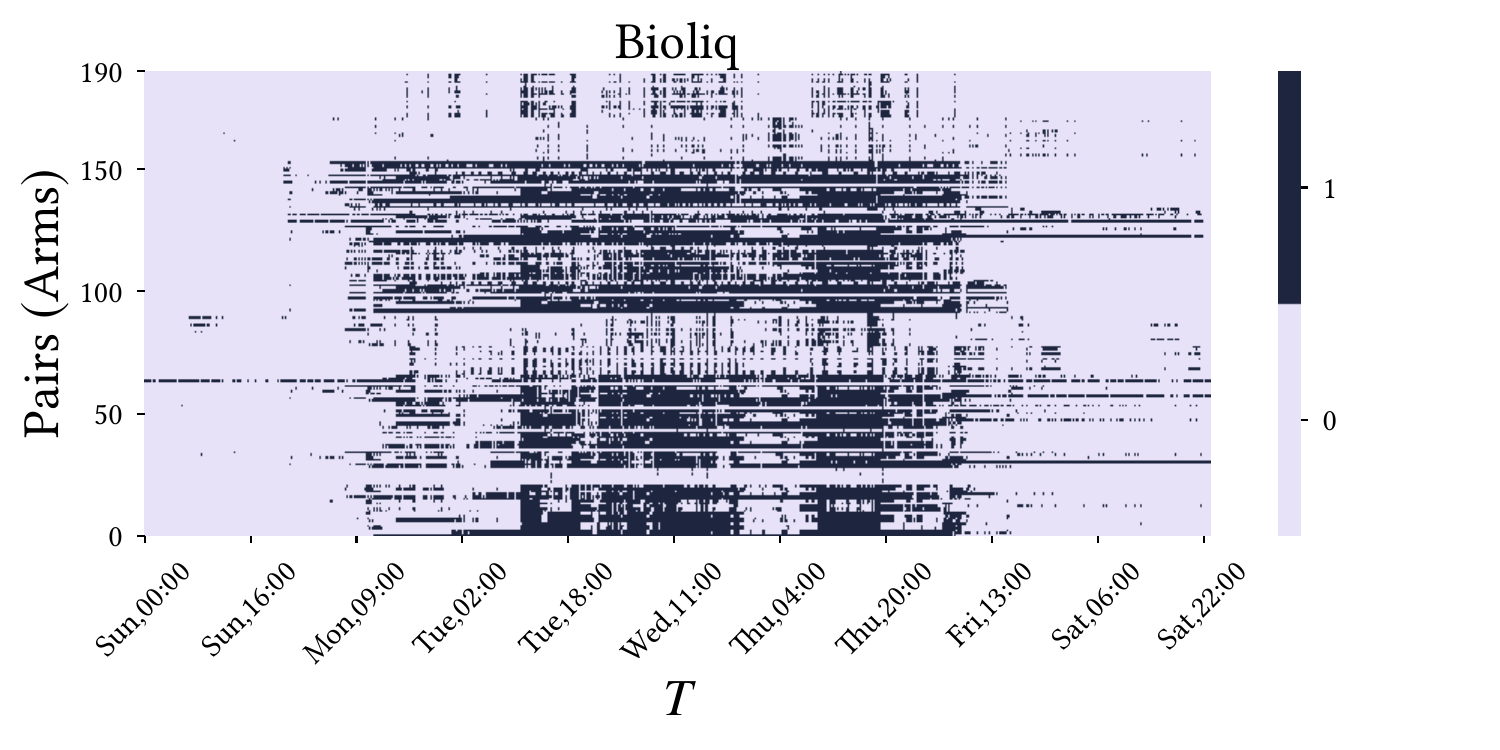}
\end{center}
	\caption{In the Bioliq dataset, each arm corresponds to the correlation coefficient between a pair of sensors. As we can see, the rewards between different arms tend to change coordinatedly. 
	We describe this dataset in details in Section~\ref{sec:experiments}.}
	\label{fig:bioliq_heatmap}
\end{figure}

This paper solves these issues by introducing adaptive resetting bandit (\AADWIN{}-bandit) algorithms, a new class of bandit algorithms.
Several algorithms, such as UCB and TS, have been established in the stationary case; therefore, we aim to extend them without introducing any forced exploration. For this purpose, we combine them with an adaptive windowing technique. For example, \AADWIN{}-TS, which is an instance of the \AADWIN{}-bandit, combines the adaptive windowing with TS. Our method deals with a subclass of changes that we call \textit{global changes}. Intuitively, if all arms change in a coordinated manner, we can avoid forced exploration to improve performance. This type of change is natural in the predictive maintenance of Example \ref{expl:bioliq}, as illustrated in Figure \ref{fig:bioliq_heatmap}, where the nonstationarity often results from changes in the entire system.

Our analysis shows that the proposed method has optimal performance for stationary streams and comparable bounds with existing NS-MAB algorithms for abruptly changing and gradually changing environments under the global change assumption.
Table \ref{tbl:comp} compares the bounds from the proposed method with the existing types of NS-MAB algorithms.

\begin{table}[]
\small
\caption{
The table below compares the achievable performance of the existing NS-MAB algorithms, \AADWIN{}-bandit, and stationary MAB algorithms (e.g., TS and UCB). The value $d$ defines the speed of gradual changes. \AADWIN{}-bandit has optimal performance in stationary environments as well as nonstationary environments with global changes. Most existing NS-MAB algorithms only handle abrupt changes and do not provide any performance guarantees under gradual changes. Performance bounds under the gradual changes are provided in \cite{besbes2014optimal,Allesiardo2015,DBLP:conf/nips/WeiHL16,DBLP:conf/amcc/WeiS18,DBLP:journals/jair/TrovoRG20,krishnamurthy2021}. Here, $\tilO$ is a Landau notation that ignores a polylogarithmic factor.
} 
\label{tbl:comp}
\begin{center}
\renewcommand{\arraystretch}{1.4}
\begin{tabular}{l|lll}
                          & Stationary                    & Abrupt & Gradual \\ \hline%
Existing NS-MAB & $\tilO(\sqrt{T})$ & $\tilO(\sqrt{T})$ & $\tilO(T^{1-d/3})$ \\ %
\AADWIN{}-bandit         & $\boldsymbol{O(\log T/\Deltamin)}$ & $\tilO(\sqrt{T})$ (Under GC) & $\tilO(T^{1-d/3})$ (Under GC)    \\ Stationary MAB         & $O(\log T/\Deltamin)$ & $O(T)$ & $O(T)$   \\ %
\end{tabular}
\end{center}
\end{table}

\subsection{Contributions}

We articulate our contributions as follows.
First, we provide an extensive analysis of adaptive windowing techniques. We focus, on the adaptive windowing (ADWIN) algorithm \citep{bifet2007learning}, which is still considered state-of-the-art in the data stream literature. The ADWIN algorithm performs well regarding various types of streams \citep{gama2014survey}. \cite{bifet2007learning} provided false positive and false negative rate bounds.\footnote{Although, rigorously speaking, their analysis is not correct, as we will discuss later.} However, existing analyses are not sufficient for our aim. For the analysis of bandit algorithms using ADWIN, we need an estimate on the accuracy of the estimator $\hatmu_t$. Thus, we conduct a finite-time analysis on the estimation error $|\mu_t - \hatmu_t|$ (Section \ref{sec:adwin}). As a by-product, this analysis explains why adaptive windowing methodologies perform well in many stream learning problems by bounding the error for abrupt and gradual changes.  

After formalizing the MAB problem (Section \ref{sec:MP-MAP}), 
we introduce the \AADWIN{} (ADaptive Resetting) bandit algorithm (Section \ref{sec:NS-MP-MAP-analysis}). Our study builds on a recent paper by \cite{DBLP:conf/kdd/FoucheKB19} that applies the MAB to sensor streams. Although \cite{DBLP:conf/kdd/FoucheKB19} proposed a way to combine the stationary bandit algorithms (e.g., UCB and TS) with adaptive windows, they did not provide any theoretical guarantees for the nonstationary case. We slightly modify their framework to enable rigorous analysis. 
Our work bridges the divide between these methods' practical application and theoretical comprehension. We believe our analysis and characterization of global changes present novel contributions to the literature on nonstationary bandit algorithms.

Finally, we demonstrate the performance of our proposed method concerning synthetic and real-world environments (Section \ref{sec:experiments}). The proposed method outperforms the existing nonstationary bandit methods in stationary streams and abruptly and gradually changing environments.

\section{Related Work}
\label{sec:relatedwork}

\subsection{Change detection}

Change detection is an important problem in data mining. The goal is to detect when the statistical properties (e.g., the mean) of a stream of observations change. Such changes are commonly attributed to the phenomenon known as \textit{concept drift} \citep{gama2014survey,DBLP:journals/tkde/LuLDGGZ19}, which characterizes unforeseeable changes in the underlying data distribution. Detecting such changes is crucial to virtually any monitoring tasks in streams, such as controlling the performance of online machine learning algorithms \citep{bifet2007learning} or detecting correlation changes \citep{DBLP:conf/iisa/SeliniotakiTCT14}. 

There are numerous methods to detect changes. The fundamental idea is to measure whether the estimated parameters of the current data distribution (e.g., its mean) have changed at any time. In other words, change detection is about separating the signal from the noise \citep{DBLP:conf/adma/GamaC06}. 

Change-point detection approaches can be classified into three major categories (cf. Table II in \cite{gama2014survey}). These categories are sequential analysis approaches (e.g., CUSUM \citep{10.2307/2333009} and its variant Page-Hinkley (PH) testing \citep{10.2307/2334386}), statistical process control (e.g., the Drift Detection Method (DDM) \citep{DBLP:conf/sbia/GamaMCR04} and its variants), and monitoring two distributions (e.g., ADWIN \citep{bifet2007learning}).

In this work, we primarily consider ADWIN \citep{bifet2007learning} because it works with any bounded distribution and has a good affinity with online learning analyses. Moreover, ADWIN monitors the mean from a sequence of observations over a sliding window of adaptive size. When ADWIN detects a change between two subwindows, the oldest observations are discarded. Otherwise, the window grows indefinitely. 
The success of this approach is due to the quasi-absence of parameters, making it highly adaptive.

\cite{DBLP:journals/eswa/GoncalvesSBV14} empirically compared ADWIN with other drift detectors. They found that ADWIN is one of the fastest detectors and is one of the only methods to provide false positive and false negative (fp/fn) rate guarantees. Many of the existing methods do not provide any performance guarantees, which is the case for DDM \citep{DBLP:conf/sbia/GamaMCR04}, EDDM \citep{baena2006early}, and ECDD \citep{DBLP:journals/prl/RossATH12} (per \cite{DBLP:journals/tkde/BlancoCRBDM15}). Although we believe that one can provide fp/fn rate guarantees for the statistical process control and CUSUM \citep{10.2307/2333009} approaches by choosing the appropriate parameters, limited discussions are available on the theoretical properties of these algorithms. For a thorough history and comparison of change-detection methods, we refer to the surveys by \cite{gama2014survey,DBLP:journals/inffus/KrawczykMGSW17}.

\subsection{Nonstationary bandits}

In comparison, theoretical performance guarantees (i.e., regret bounds) are much more emphasized in the MAB literature. Traditionally, NS-MAB algorithms are divided into two categories. \textit{Active} methods actively seek to detect changes, and \textit{passive} methods do not. Another type of algorithms, known as adversarial bandits \citep{auer1995gambling}, such as Exp3, can deal with changing environments. However, the guarantee of adversarial bandit algorithms is limited when no arm is consistently good. In the following paragraphs, we discuss the related work on active and passive methods.

\textit{Active methods:} \cite{hartland:inria-00164033} proposed Adapt-EvE, which combines the PH test with UCB, and \cite{DBLP:conf/aistats/MellorS13} suggested change-point TS. However, these two studies do not provide any regret bounds. \cite{DBLP:conf/aaai/LiuLS18} proposed CUSUM-UCB and PH-UCB, which combine the UCB algorithm with a CUSUM-based (or PH-based) resetting and forced exploration. They derived a regret bound of $\tilO(\sqrt{MT})$ in an abruptly changing environment with $M$ known change points. \cite{DBLP:conf/aistats/0013WK019} suggested the M-UCB algorithm, which combines UCB with an adaptive-window-based resetting method with a regret bound of $\tilO(\sqrt{MT})$ for an abruptly changing environment. 
\cite{Allesiardo2015} proposed Exp3.R, which combines Exp3 \citep{auer1995gambling} with change-point detection and provides a $\tilO(M\sqrt{T})$ regret bound. Although the assumptions are slightly different among algorithms, many of the regret bounds are of the order of $\tilO(\sqrt{MT})$ concerning the number of change points $M$ and the number of time steps $T$. \cite{DBLP:conf/colt/AuerGO19} provided an epoch-based bandit algorithm that does not require the knowledge of $M$. Moreover, \cite{DBLP:conf/aistats/SeznecMLV20} applied the adaptive window to the rotting bandit problem where the reward of the arms only decreases.
\updated{
\cite{krishnamurthy2021}
proposed an elimination-based algorithm and analyzed its performance in a two-armed gradually changing environment. They derived a similar regret bound $O(T^{1-d/3})$ as this paper, where $d$ is
the change-speed parameter that is the same as ours.
}
\cite{besson2020efficient} proposed GLR-klUCB, a combination of KL-UCB \citep{garivierklucb} with the Bernoulli Generalized Likelihood Ratio (GLR) test \citep{maillardalt2017} and forced exploration for an abruptly changing environment. They discussed global and local resets and provided a regret bound of $\tilO(\sqrt{MT})$.\footnote{They also discussed dependence on the distribution-dependent constants (Corollary 6 therein).}
\updated{%
One of the closest works to ours is 
\cite{Mukherjee2019}, where 
they proposed algorithms UCBL-CPD and ImpCPD by combining UCB with a change point detector.
They considered the global changes in the abrupt environment
with the same spirit as ours to avoid forced exploration.
Still, their main goal is to characterize detailed regret bounds for this environment, while ours is to cover both the global abrupt and global gradual environments by a single algorithm.
To be more specific, they obtain both distribution-dependent and independent bounds using a confidence interval established through the Laplace method. This approach circumvents the need for union bounds, leading to a reduced confidence level.\footnote{In particular, confidence levels required in Theorems 1 and 2 therein are $1/T$ to the number of rounds $T$. This is superior to the $T^3$ level of confidence required in our paper.}
In terms of distribution-independent regret, ImpCPD has $O(\sqrt{MT})$ distribution-independent regret bound, while UCBL-CPD and our algorithm have $\tilde{O}(\sqrt{MT})$ regret including the polylogarithmic factor in $T$.
On the other hand,
most of the effort of this paper is devoted to the analysis for the gradual changes.
We think that the application of the Laplace method could potentially improve polylogarithmic factors of our bounds, but we opted to use confidence bounds based on union bounds for simplicity to focus on simultaneously handling two environments.
}

\textit{Passive methods:} 
\cite{kocsis2006discounted} proposed Discounted UCB (D-UCB). \cite{Garivier2008} suggested Sliding Window UCB (SW-UCB) and analyzed D-UCB to demonstrate that the two algorithms have a $\tilO(\sqrt{MT})$ regret bound for abruptly changing environments. \cite{besbes2014optimal} considered the case of limited variation $V$ and proposed the Rexp3 algorithm with the regret bound $\Theta(V^{1/3}T^{2/3})$. \cite{DBLP:conf/nips/WeiHL16} generalized the analysis of \cite{besbes2014optimal} to the case of intervals where each interval has limited variation. \cite{DBLP:conf/amcc/WeiS18} proposed the LM-DSEE and SW-UCB\# algorithms with a $\tilO(\sqrt{MT})$ regret bound for an abruptly changing environment. The latter algorithm adopts an adaptive window with a limited length. They also analyzed the case of a slowly changing environment and derived a regret bound of SW-UCB\#.
\cite{DBLP:conf/colt/ChenLLW19} proposed a very involved algorithm and stated that it can adapt to abrupt changes as well as gradual changes with a $\Theta(\min(\sqrt{MT}, V^{1/3}T^{2/3}))$ regret bound.
\cite{DBLP:journals/jair/TrovoRG20} proposed the Sliding Window TS (SW-TS) algorithm and provided a $\tilO(\sqrt{MT})$ regret bound for an abruptly changing environment. They also provided a distribution-dependent regret bound for a gradually changing environment. %

Our approach is an active one; however, it does strikingly differ from the existing methods. Our algorithm does not sacrifice performance in the stationary case, whereas almost all existing NS-MAB algorithms are exclusively designed for nonstationary environments. In real-world settings, users may not know whether a given stream of data must be considered stationary or not. Similarly, the nature of the stream may also change over time. In such settings, our approach is advantageous, as our experiments show. 

Although the analysis that follows is quite involved, our algorithms are conceptually simple and aimed at practical use. In comparison, our competitors tend to require more computation and rely on parameters that are difficult to set.

\section{Analysis of ADWIN}
\label{sec:adwin}

This section analyzes ADWIN \citep{bifet2007learning}. This algorithm monitors at any time $t$ an estimate of the mean $\hatmu_t$ from a single stream of univariate observations.

\subsection{Data streams} 
\label{subsec:problem_setting}

We assume that each observation $x_t \in [0,1]$ at any time step $t$ is drawn from some distribution with the mean $\mu_t$. The value $\mu_t$ is not known to the algorithm, and ADWIN estimates it from a sequence of (possibly noisy) observations $S:(x_1, x_2, \dots, x_T)$. The goal of ADWIN is to maintain an estimator of $\mu_t$ at any $t$ based on past observations. Due to \textit{concept drift} \citep{gama2014survey}, the mean $\mu_t$ may change over time, so the task is not trivial. 

\begin{figure}
	\centering
		\includegraphics[width=\linewidth]{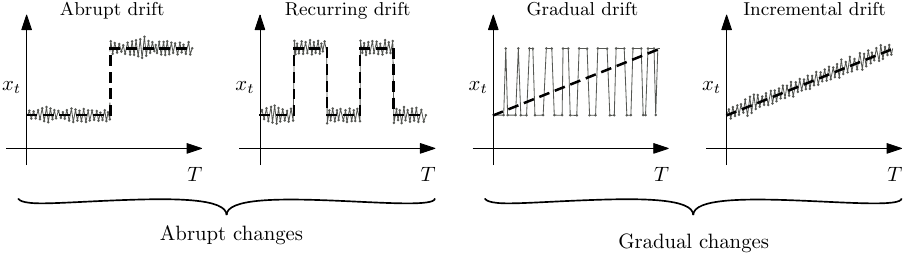}
	\caption{We classify the four types of concepts drifts (abrupt, recurring, gradual, incremental) into two types of changes: abrupt and gradual changes. Here, dashed lines represent $\mu_t$, whereas solid lines represent $x_t$.}
	\label{fig:conceptdrifttypes}
\end{figure}

The data stream literature \citep{gama2014survey} identifies the four main categories of concept drifts: abrupt, recurring, gradual, and incremental drifts. Figure \ref{fig:conceptdrifttypes} indicates that the mean $\mu_t$ (the dashed line in the figure) changes \textit{abruptly} and stays the same for some time with abrupt and recurring drifts, whereas $\mu_t$ changes \textit{gradually} over time with gradual and incremental drifts. The actual observations $x_t$ can be arbitrarily noisy. We are only interested in the change in the mean $\mu_t$; thus, we simplify our analysis to the two types of changes.\footnote{This categorization matches the MAB literature \citep{DBLP:conf/amcc/WeiS18,DBLP:journals/jair/TrovoRG20}.} 
In addition, we also consider stationary streams where the mean $\mu_t$ never changes over time. In summary, we consider three types of streams:
\begin{mydef}{\rm (Stationary, abruptly changing , and gradually changing  streams)}\label{def:streams}
\begin{enumerate}
\item A stream is {\it static} if $\mu_t = \mu$ for all $t\in[T]$ and some $\mu\in[0,1]$. 
\item A stream is {\it abruptly-changing} if $\mu_t = \mu_{t+1}$ except for changepoints $\ET_C \subset [T]$.
\item A stream is {\it gradually-changing} if $|\mu_{t+1} - \mu_t| \le b$ for all $t\in[T]$ and some constant $b \in (0,1)$.

\end{enumerate}
These definitions refer to the mean $\mu_t$. The observations $\{x_1,x_2,\dots\}$ can be noisy.
\end{mydef}

\subsection{The adaptive window algorithm}
\label{subsec:adwin_alg}

\begin{algorithm}
	\caption{ADWIN}\label{alg:ADWIN} 
	\begin{algorithmic}[1]
		\Require A univariate stream of values $S: (x_1,x_2,\dots) \in [0,1]$, confidence level $\delta \in (0,1)$.
		\State $W(1) = \{\}$ %
		\For{$t = 1,2,\dots$}
		\State $W(t+1) = W(t) \cup \{t\}$
		\While{$|\hatmu_{W_1} - \hatmu_{W_2}| \ge \epscut^\delta$ holds for some split $W(t+1) = W_1 \cup W_2$}
		\State $W(t+1) = W_2$.
		\EndWhile 
		\EndFor
	\end{algorithmic}
\end{algorithm}

Algorithm \ref{alg:ADWIN} is the pseudo-code for ADWIN \citep{bifet2007learning}. %
The law of large numbers implies that using more observations helps estimate the parameter. However, the nature of older observations might be different from more recent observations. To determine a good trade-off between these two effects, ADWIN maintains a window $W(t)$ of past time and discards data points outside the window.  

We omit the index $(t)$ of $W(t)$ when the time step of interest is clear. At each time step $t$, ADWIN receives a new observation $x_t$ and extends the window $W$ to include the observation. For the current window $W$, we let $\hatmu_W = \frac{1}{|W|} \sum_{t \in W} x_t$ be the corresponding empirical mean. For each new observation, ADWIN tests whether the mean of the underlying distribution has changed. If we can split $W$ into two consecutive disjoint subwindows $W_1 \cup W_2 = W$ whose empirical means are significantly different (i.e., by some threshold $\epscut$), then ADWIN discards $W_1$ (i.e., ADWIN “shrinks” the window): 
\begin{align} %
\epscut^\delta = \sqrt{\frac{1}{2|W_1|} \log\left(\frac{1}{\delta}\right)} + \sqrt{\frac{1}{2|W_2|} \log\left(\frac{1}{\delta}\right)}, %
\label{def_epscut}
\end{align} %
where $|W|$ is the cardinality of a set $W$.
The threshold $\epscut$ is based on Hoeffding's inequality\footnote{
Note that, $\epscut$ above is slightly different from the original ADWIN
where the harmonic mean of $|W_1|$ and $|W_2|$ is used.}. 
At each time step, ADWIN check every possible split of $W$.

\begin{figure}
	\centering
	\hfill
	\subfigure[changepoint]{
		\label{fig:changepoint}
		\includegraphics[width=0.30\linewidth]{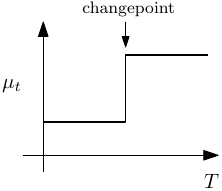}
	}
	\hfill
	\subfigure[detection time]{
		\label{fig:break}
		\includegraphics[width=0.30\linewidth]{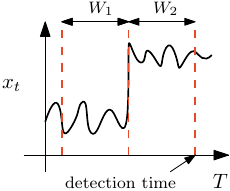}
	}
	~
	\hfill
	\hfill
	
	\caption{Terminology of changepoint and detection time. A changepoint of an abruptly changing stream is a time step $t$ where $\mu_t \ne \mu_{t+1}$. A detection time is the time step where ADWIN shrinks the window.} 
	\label{fig:breakpoint_adwin}
\end{figure}

The following definition formalizes the notion of a window and the shrinking of the current window. See also Figure \ref{fig:breakpoint_adwin} for illustration.   
\begin{mydef}{\rm (Detection times)}\label{def:shrinktime} 
The time step $t$ is a detection time of ADWIN if $|W(t+1)| \leq |W(t)|$. In addition, ADWIN ``shrinks'' the window at time $t$ if time step $t$ is a detection time. 
We define the breakpoint as the last round of $W_2$ from the split $W = W_1 \cup W_2$ (Figure \ref{fig:breakpoint_adwin}). 
We let $\ET_d$ be the set of all detection times.
\end{mydef}
\begin{myrem}{\rm
Unlike the set of changepoints $\ET_c$ (Definition \ref{def:streams}) that is only defined for an abruptly changing stream and is independent of any underlying algorithm, $\ET_d$ is defined for both abrupt and gradual stream. The detection time is a random variable that is defined via ADWIN.
}\end{myrem}

\subsubsection{Bound for the estimator of the mean $\mu_t$}

\cite{bifet2007learning} derived a bound on the false positive and false negative rates by replacing $\delta$ with $\delta' = \delta / |W(t)|$, to account for multiple tests\footnote{There are $|W(t)|-1$ ways to split $W(t) = W_1 \cup W_2$.}. 
However, the existing analysis is incomplete
because it implicitly assumes that the window size $|W(t)|$ is deterministic.

In this paper, we use $\delta$ instead of $\delta'$ because it clarifies the analysis. The value of $\delta$ is in accordance with the possible multiplicity. With this aim, we introduce the notion of the window set and split and then introduce a confidence bound that holds with high probability.

\begin{mydef}{\rm (Window set)}\label{def:window}
The window set $\EW$ is the set of all the segments, which are the candidates of the current window $W$. 
\begin{equation}
\EW = \{W': W' = \{t',t'+1,\dots,t''\}, 1 \le t' < t'' < T\}.
\end{equation}
\end{mydef}%
\begin{mydef}{\rm (Split and estimator in a window)}
Letting $W' = \{t',t'+1,\dots,t''\}$, 
a split $W_1, W_2$ of a window $W' \in \EW$ is defined as two disjoint subsets of $W'$ such that $W_1 = \{t',t'+1,\dots,m\}$ and $W_2 = \{m+1,m+2,\dots,t''\}$ for some $m \in W' \setminus \{t''\}$.
For a window $W' \in \EW$,
\begin{equation}
\mu_{W'} = \frac{1}{|W'|} \sum_{s \in W'} \mu_s,
\end{equation}
and its empirical estimator is
\begin{equation}
\hatmu_{W'} = \frac{1}{|W'|} \sum_{s \in W'} x_s.
\end{equation}
\end{mydef}
\begin{proposition}{\rm (Uniform Hoeffding bound for the window set)}\label{prop:allwindowsbound}
Let $p>0$ be arbitrary. With probability $1 - 2/T^p$, we have the following:
\begin{equation}\label{ineq:allwindowsbound}
\forall\ {W' \in \EW}, \quad |\mu_{W'} - \hatmu_{W'}| \le \sqrt{\frac{\log (T^{2+p})}{2 |{W'}|}}.
\end{equation}
\end{proposition}
Proposition \ref{prop:allwindowsbound} is derived using the Hoeffding inequality (Lemma \ref{lem:hoeffding}, in Appendix \ref{sec:lemmas}) over all possible $|\EW| \le T^2$ windows.
The bound above holds regardless of the randomness of the current window $W = W(t)$. Moreover, it holds for any $W$ and any split $W_1 \cup W_2 = W$. 
We typically set $p=1$ because this ensures a uniform bound with probability $1 - O(1/T)$, and an event with probability $O(1/T)$ is usually negligible. 
Our results are based on this bound.

\subsubsection{Bound for the total error of the mean estimator}

In the following, we provide a bound on the following total error of the estimator $\hatmu_W$: 
\begin{mydef}{\rm (Total error)}\label{def:totalerror}
The total error of estimator $\hatmu_W$ is defined as follows:
\begin{equation}
\Err(T) = \sum_{t=1}^T |\hatmu_{W(t)} - \mu_t|.
\end{equation}
\end{mydef}

We first derive the error bound of ADWIN for a stationary stream, which directly follows from the false positive rate (Eq.~\eqref{ineq:allwindowsbound}).

\begin{thm}{\rm (ADWIN in stationary environments)}\label{thm:adwinstationary}
Let the stream be stationary. Then, for ADWIN with $\delta = 1/T^3$, we have the following:
\begin{equation}
\Ex[\Err(T)] \le \tilO(\sqrt{T}).
\end{equation}
\end{thm}
\begin{proofof}{thm:adwinstationary} %
Equation \eqref{ineq:allwindowsbound} with $p=1$, together with Definition \ref{def:shrinktime} implies that no shrink occurs with probability $1 - 2/T$. 
Therefore,
\begin{align}
\Ex[\Err(T)] 
&\le 
\underbrace{T \times \frac{2}{T}}_{\text{case where at least one shrink occurs}} + \underbrace{\sum_{t=1}^T \sqrt{\frac{\log ( T^3)}{2t}}}_{\text{case where no shrink occurs}}\\
&\le 1 + \sum_{t=1}^T \sqrt{\frac{\log (T^3)}{2t}} \\
&\le \sqrt{6 T(\log T)} + 1, \quad \text{\ \ \ $\left(\text{by} \sum_{t=1}^T (1/\sqrt{t}) \le 2\sqrt{T}\right)$}
\end{align}
which completes the proof.
\end{proofof}
The $\tilO(\sqrt{T})$ error is the optimal rate because we can only identify the true value of $\mu_W$ up to $O(\sqrt{1/|W|})$. The error per time step $|\mu_t - \mu_W|$ is at least $\Omega(1/\sqrt{|W|}) = \Omega(1/\sqrt{t})$,
and the total error is $\Omega(\sum_t 1/\sqrt{t}) = \Omega(\sqrt{T})$. 

Theorem \ref{thm:adwinstationary} states that ADWIN can learn from a stationary environment without any unnecessary shrinking. This statement contrasts with such methods as periodic resetting or fixed-size windowing algorithms which discard their entire memory after a fixed period. 
Having derived the learnability of ADWIN for a stationary stream, we are now interested in the property of ADWIN in the face of a nonstationary stream.

\subsection{Analysis of ADWIN for abrupt changes}\label{subsec:adwin_abrupt}

This section derives an error bound of ADWIN in the face of an abrupt change. As informally discussed in \cite{bifet2007learning}, 
ADWIN is able to detect abrupt changes if the changes are infrequent and gaps are detectable. Our results here are even more robust. Somewhat surprisingly, the bound in Theorem \ref{thm:adwinabrupt}
does not depend on the detectability of the change. No matter how large or small the changes are and in which interval the changes occur, the algorithm's performance is bounded in terms of the number of changes. 

\begin{thm}{\rm (Error bound of ADWIN under abrupt changes)} \label{thm:adwinabrupt}
Let the environment be abrupt with $\NumChange$ changepoints.
The total error of ADWIN with $\delta = 1/T^3$ is bounded as \begin{equation}\label{ineq:adwinabruptbound}
\Ex[\Err(T)] = \tilO(\sqrt{MT}).
\end{equation}
\end{thm}
\begin{sketch}%
Let $c(t)$ be the number of time steps after the last changepoint. 
First, we demonstrate that $|\mu_W - \mu_t| = \tilO(1/\sqrt{c(t)})$ because ADWIN otherwise would shrink the window further (i.e., the event $\EB$ in the proof).
Second, the window size $|W|$ is also bounded below. Given a sufficiently high confidence parameter, $|W| = O(c(t))$ and thus $|\mu_W - \hatmu_W| = \tilO(1/\sqrt{c(t)})$. Combining these two yields the bound of \begin{equation}
|\hatmu_W - \mu_t| \le |\mu_W - \mu_t| + |\mu_W - \hatmu_W| = \tilO(1/\sqrt{c(t)}).
\end{equation}
Using the Cauchy-Schwarz inequality yields the bound of Eq.~\eqref{ineq:adwinabruptbound}.
\end{sketch}
The formal proof is found in Appendix \ref{subsec:proof_adwinabrupt}.
\begin{myrem}{\rm
Theorem \ref{thm:adwinabrupt} implies the optimality of ADWIN under abrupt drift. To observe this, assume that a changepoint exists every $T/M$ time steps. As discussed in Theorem \ref{thm:adwinstationary}, the optimal rate of error for each interval between changepoints is $\tilTheta(\sqrt{T/M})$, and the total error should be $\tilTheta(M \times \sqrt{T/M} ) = \tilTheta(\sqrt{MT})$, which matches Theorem \ref{thm:adwinabrupt}.
}\end{myrem}

\subsection{Analysis of ADWIN for gradual changes}\label{subsec:adwin_gradual}
In this section, we analyze ADWIN for a gradually changing stream, where the mean $\mu_t$ changes slowly with constant $b$ (Definition \ref{def:streams}). We consider  the error for $b=T^{-d}$ by following the framework of \cite{DBLP:conf/amcc/WeiS18}.

\begin{thm}{\rm (Error bound of ADWIN under gradual changes)}\label{thm:adwingradual}
Assume that there exists $d \in (0,3/2)$ such that, the stream is gradually changing with its constant satisfies $b \le T^{-d}$.
Then, the performance of ADWIN with $\delta \le 1/T^3$ is bounded as follow:
\begin{equation}
\Ex[\Err(T)] = \tilO(T^{1-d/3}).
\end{equation}
\end{thm}
\begin{sketch}%
We establish two lemmas in the appendix. Lemma \ref{lem:gradualbound} states that the drift $|\mu_s - \mu_t|$ for any two time steps $s,t$ in the current window $W(t)$ is bounded by $\tilO(b N + \sqrt{1/N})$ for any $N \ge |W(t)|$. Moreover, Lemma \ref{lem:errorshrink} states that the window is likely to grow until $|W| = O(b^{-2/3})$. Combining these two lemmas yields $|\mu_s - \mu_t| = O(b^{1/3}) = O(T^{-d/3})$.
\end{sketch}%
The formal proof is found in Appendix \ref{subsec:proof_adwingradual}.
Theorem \ref{thm:adwingradual} states that, if the change is slow compared with the current scale of interest $T$ then the error per time step $\Err(T)/T$ approaches zero. 

In this section, we have bounded the total error for streams with abrupt changes (Section \ref{subsec:adwin_abrupt}) as well as for streams with gradual changes (Section \ref{subsec:adwin_gradual}). This concludes our discussion on ADWIN.
In subsequent sections, we consider the idea of combining ADWIN with the MAB setting where multiple streams are involved, and only a selected subset of streams are observable.

\section{Multi-armed Bandits}
\label{sec:MP-MAP} %

\begin{algorithm} 
	\caption{Thompson sampling (TS)}
\label{alg:mpts} 
	\begin{algorithmic}[1]
		\Require Set of arms $[K]$. 
		\State \textbf{Initialize:} $W(t) = \emptyset$ \label{line:mpts_init}
		\For{$t = 1,\dots,T$}  \label{line:mpts_loop}
    		\For{$i = 1,\dots,K$}  
    		    \State $\Sit = \sum_{s \in W(t)} \Ind[i = I(s)]x_{i,s}$, $\Nit = \sum_{s \in W(t)} \Ind[i = I(s)]$.
        		\State $\tilmu_{i,t} \sim \text{Beta}(\Sit+1, \Nit-\Sit+1)$. \label{betadrawing}
    		\EndFor 
    		\State Play arm $I(t) := \argmax_{i} \tilmu_{i,t}$. \Comment{Posterior sampling} \label{sorting-theta}
    		\State Receive reward $x(t)$.
    		\State Update window $W(t+1) = W(t) \cup \{t\}$.
        \EndFor
\end{algorithmic}
\end{algorithm}
\begin{algorithm} 
	\caption{Kullback–Leibler UCB (KL-UCB).}\label{alg:mpklucb} 
	\begin{algorithmic}[1]
		\Require Set of arms $[K]$. 
		\State \textbf{Initialize:} $W(t) = \emptyset$
		\For{$t = 1,\dots,T$}  
    		\For{$i = 1,\dots,K$}  
    		    \State $\Sit = \sum_{s \in W(t)} \Ind[i = I(s)]x_{i,s}$, $\Nit = \sum_{s \in W(t)} \Ind[i = I(s)]$.
        		\State $\hatmu_{i,W(t)}= \Sit/\Nit$ where $0/0 = 1$.
        		\State $U_i(t) = \max \{q \in [0,1]: \Nit \KL(\hatmu_{i,W(t)},q)\le \log( t/\Nit )\}$. \Comment{KL-UCB index} \label{klucbindex}
    		\EndFor 
    		\State Play arm $I(t) := \argmax_{i}  U_i(t)$. \Comment{Arm with the largest UCB index} \label{sorting-ucb}
            \State Receive reward $x(t)$.
    		\State Update window $W(t+1) = W(t) \cup \{t\}$.
        \EndFor
	\end{algorithmic}
\end{algorithm}

Up to now, we have studied the single-stream learning problem, in which every values of the stream are observed. From this section, we study the Multi-Armed Bandit (MAB) problem \citep{Thompson1933,robbins1952}. 
This problem involves $K$ streams, and a forecaster can only observe one of these streams at each time step.
We first formalize the problem setting we consider. 

\komiyama{Changed notation $I(t)$. Before: set (multiple play). Now: single arm.}
Let there be $K$ arms. At each time step $t = 1, \dots, T$, the forecaster selects an arm $I(t) \in [K]$, and then receives a reward $x_{I(t),t}$ from the selected arm. 
We also use the term ``environment'' to describe $K$ streams that generates rewards $x_{i,t}$. Let $(\mu_{i,t})_{t=[T]}$ be the $i$-th stream.
The stationary (resp.~abruptly-changing and gradualy-changing) environment is defined as a set of $K$ stationary (resp.~abruptly-changing and gradualy-changing) streams, respectively, where these streams are defined in Definition \ref{def:streams}.

The goal of the forecaster is to maximize the sum of the rewards of the selected arms by using a good algorithm,
 and the performance of a bandit algorithm is usually measured by the regret, which is defined as the difference between the expected reward of the best arm and the expected reward of the arms selected by the algorithm.
That is, 
\begin{equation}
\reg(t)
=
\max_{i\in[K]}\mu_{i,t} - \mu_{I(t),t}
\end{equation}
and 
\begin{equation}\label{ineq:regret}
\Reg(T)
=
\sum_{t=1}^T \reg(t).
\end{equation}

In the $K$-armed bandit problem, Thompson sampling \citep[TS]{Thompson1933} and the upper confidence bound \citep[UCB]{lairobbins1985,auer2002finite} are widely known algorithms that balance exploration and exploitation. Among several variants of UCB, the one that utilizes KL divergence, which is called KL-UCB \citep{garivierklucb}, is known to have a state-of-the-art performance in stationary environments. TS and KL-UCB in our notation are described in Algorithms \ref{alg:mpts} and \ref{alg:mpklucb}, respectively. Here, $\Ind[\EA]=1$ if $\EA$ or $0$ otherwise.

Under a changing environment, the performance of these algorithms is no longer guaranteed. 
The next section extends the MAB framework for the case of abruptly-changing and gradually-changing streams, by combining adaptive window technology with bandit algorithms.

These algorithms are designed to deal with stationary environments: The regret of KL-UCB and TS is bounded as follows:
\komiyama{(Note: This was "DD" (distribution-dependent) in the previous manuscript)}
\begin{mydef}{\rm (Stationary regret)}\label{def:logarithmic}
For a stationary environment (i.e., $\mu_{i,t} = \mu_i$), for ease of discussion, we assume $\mu_1 > \mu_2 > \dots > \mu_K$. Let the suboptimality gap be $\Delta_i = \mu_1 - \mu_i$.
A bandit algorithm has a logarithmic stationary regret if a universal constant $\Cst$ exists such that 
\begin{equation}
\updated{
\Ex[\Regbase(T)] \le \Cst \sum_{i \ne 1} \frac{\log{T}}{\Delta_i},}
\end{equation}
where $\Regbase(T)$ denotes the regret when we run the bandit algorithm.
\end{mydef}
In the literature of multi-armed bandit problem, this bound is often referred to as the distribution-dependent bound in the sense that it depends on $\Delta_i$. The inverse of $\Delta_i$ defines the hardness of the instance. Note that the logarithmic bound and its first-order dependence on $\Delta_i^{-1}$ is optimal \citep{lairobbins1985,auer2002finite,kaufmann2012thompson,agrawalaistats13}. 

\subsection{Drift-tolerant regret}\label{subsec:adwinbanditanalysis}

While the stationary regret is the most widely studied measure of the performance of stationary bandits, it cannot characterize the performance of a bandit algorithm under a changing environment.
The following extends this to the environments with drifts.
\komiyama{Updated. Now this is the form of distribution-dependent, and in expectation (no high-prob ver)}
\begin{mydef}{\rm (Drift-tolerant regret)}\label{def:driftrobust}
Assume a nonstationary environment that is abruptly or gradually changing.
Let
\begin{equation}
\Delta_i = \mu_{1,1} - \mu_{i,1}
\end{equation}
be the gap at $t=1$.
Let
\komiyama{06/16/2023 updated. Previous: $\eps(t) = \sum_{s<t} \max_i |\mu_{i,s} - \mu_{i,s+1}|$. Not much difference?}
\begin{equation}
\eps(t) = \max_{s\le t} \max_i |\mu_{i,s} - \mu_{i,1}|,
\label{def_epst}
\end{equation}
which is the maximum drift of the arms by time step $t$.
For $c>0$, let
\[
\Regbasetor(T, c) := \sum_t (\reg(t) - c\eps(t))_+,
\]
where $(x)_+ = \max(x, 0)$. Namely, $\Regbasetor(T, c)$ is the regret where the regret proportional to drift is tolerated.
A bandit algorithm has logarithmic drift-tolerant regret if a factor %
$\Cdr = O(1)$ 
exists such that
\begin{equation}
\Ex[\Regbasetor(T, \Cdr)] \le \Cdr
\sum_{i \ne 1} \frac{\log T}{\Delta_i}.
\end{equation}
\end{mydef}
\begin{myrem}{\rm
Definition \ref{def:driftrobust} is a generalization of the distribution-dependent regret in the literature of stationary bandits that allows a drift proportional to $\eps(t)$. Letting $\eps(t)=0$ immediately derives the stationary regret of Definition \ref{def:logarithmic}.
}\end{myrem}

As the following lemma characterizes, Thompson sampling (TS, Algorithm \ref{alg:mpts}) is drift-tolerant.
\begin{lem}\label{lem_tsdt}
Assume that rewards $x_{1,t},x_{2,t},\dots$ are binary (i.e., $0$ or $1$).
Then, TS has logarithmic drift-tolerant regret.
\end{lem}%
The formal proof is in Appendix \ref{subsec:proof_ts}.
Note that it is not very difficult to derive the drift-tolerant regret of KL-UCB by following similar steps as the proof of Lemma \ref{lem_tsdt}.

Although the drift-tolerant regret characterizes the performance of TS and KL-UCB under nonstationary environments, these algorithms are not good enough to deal with nonstationarity. This is because $\epsilon(t)$ increases over time and can lead to a meaningless bound of $\Theta(T)$ when $\epsilon(t)$ reaches $\Theta(1)$. A capable nonstationary bandit algorithm should be able to forget past data when it identifies a substantial drift. In the following section, we will merge these \text{base-bandit} algorithms with ADWIN.

\section{Analysis of MAB for Globally Nonstationary Environments}
\label{sec:NS-MP-MAP-analysis}

This section proposes the combination of adaptive windowing and bandit algorithms.  

\subsection{Adaptive resetting bandit algorithm}\label{subsec:aadwinbandit}

\begin{algorithm}  
	\caption{\ADWIN-bandit}\label{alg:ADWINbandit} 
	\begin{algorithmic}[1]
		\Require Set of arms $[K]$, confidence level $\delta$, base-bandit algorithm
        \State Initialize the base-bandit algorithm.
		\For{$t = 1,2,\dots,T$}
    		\State $(I(t), X(t)) = $ \Call{Base-bandit}{$W(t)$} \label{basebandit_round_ads} \Comment{Do one time step of the base-bandit algorithm}
    		\If{there exists a split $W(t+1) = W_1 \cup W_2$ such that $|\hatmu_{i,W_1} - \hatmu_{i,W_2}| \ge \epscut^\delta$} \label{changedetect_ads} 
                \State Update the window $W(t+1) = W_2$ of the base-bandit algorithm.
            \EndIf
		\EndFor
	\end{algorithmic}
\end{algorithm}

\begin{algorithm}
	\caption{\AADWIN-bandit}\label{alg:AADWINbandit} 
	\begin{algorithmic}[1]
		\Require Set of arms $[K]$, confidence level $\delta$, monitoring parameter $N\in\mathbb{N}$, base-bandit algorithm
    \State Initialize base-bandit algorithm.
    \For{$l = 1,2,\dots,\lceil \log_2 (T/\wasD N+1)\rceil $}\label{adwin_blocks}
		\For{$t = (2^{l-1}-1)\wasD N+1,\dots,\min\{(2^{l}-1)\wasD N,T\}$}\label{loop_inner} 
    		\If {$l \ge 2$ and $t = 0 \pmod{\wasD }$}
     		\State Pull $i^{(l-1)}$. \label{monitoring_one} \Comment{Monitoring arm of the previous block.}
    		\ElsIf {$l \ge 2$ and $t = 1 \pmod{\wasD }$ and $t\ge (2^{l-1}-2)\wasD N+1$}
      \State Pull $i^{(l)}$. \label{monitoring_two}
            \Comment{Monitoring arm of the current block.}
            \Else
                \State $(I(t),X(t)) = $ \Call{Base-bandit}{$W(t)$} \label{basebandit_round} \Comment{Do one step of the base-bandit algorithm}
            \EndIf
    		\If{there exists $W(t+1) = W_1 \cup W_2$ such that
		$|\hatmu_{i,W_1} - \hatmu_{i,W_2}| \ge \epscut^\delta$} \label{changedetect} 
        		\State Reset the entire algorithm with $T:=T-t$.  \label{initstart_monitoring} 
        \EndIf
        \If{$t= \wasD N + 1$} \Comment{The end of the first block $l=1$.}
            \State Set $i^{(1)} = \argmax_{i \in [K]} N_i^{(1)}$ \Comment{See Eq.~\eqref{ineq_monitordef} for $N_i^{(1)}$}.  \label{line_monitoring_one}
        \ElsIf{$l \ge 2$ and $t=(2^{l-1} - 2)\wasD N$} \Comment{Before the beginning of the last subblock of the block $l$.}
            \State Set $i^{(l)} = \argmax_{i \in [K]} N_i^{(l)}$. \Comment{See Eq.~\eqref{ineq_monitordef} for $N_i^{(l)}$} \label{line_monitoring_l}
        \EndIf
		\EndFor
		\EndFor
	\end{algorithmic}
\end{algorithm}%

\begin{figure}[t]
\begin{center}
    \includegraphics[width=1.\linewidth]{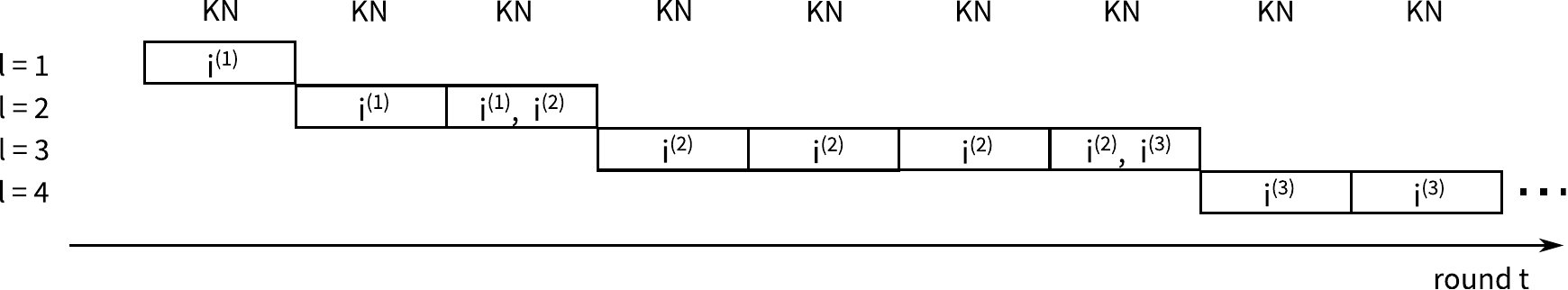}
\end{center}
\caption{
Illustration of the blocks in Algorithm \ref{alg:AADWINbandit}. Here, ``$i^{(l)}$'' for each block $l$ represents the monitoring arm, which is drawn at least $N$ times in each subblock. In the last subblock, two monitoring arms $i^{l-1}$ and $i^{l}$ are drawn at least $N$ times, for the aim of monitoring consistency.
 }
\label{fig:monitoring}
\end{figure}

\ADWIN{}-bandit (adaptive shrinking bandit; Algorithm \ref{alg:ADWINbandit}) combines adaptive windows with a \textit{base-bandit} algorithm such as TS or KL-UCB. In this algorithm, each arm has an associated adaptive window change detector. 
When such a detector identifies a split in the window $W$ into $W_1$ and $W_2$, it follows the ADWIN procedure (as explained in Section \ref{sec:adwin}) of reducing the window size to retain $W_2$.

We explain the steps of Algorithm \ref{alg:ADWINbandit} with TS as a base-bandit algorithm. Before the first time step, the base-bandit algorithm (Line \ref{line:mpts_init} of Algorithm \ref{alg:mpts}) is initialized. At each time step $t=1,2,\dots$, Line \ref{basebandit_round_ads} of Algorithm \ref{alg:ADWINbandit} runs an iteration of the base-bandit (Line \ref{line:mpts_loop} of Algorithm \ref{alg:mpts}) to obtain $I(t)$ and corresponding reward $x(t) = x_{I(t),t}$. Afterwards, Line \ref{changedetect_ads} of Algorithm \ref{alg:ADWINbandit} checks whether a change occurred for each arm. If at least one change is detected, it shrinks window $W(t+1)$ to $W_2$ and drops the memory of $W_1$.

Although we will demonstrate the superior empirical performance of Algorithm \ref{alg:ADWINbandit} in Section \ref{sec:experiments}, it is highly nontrivial to derive a regret bound of Algorithm \ref{alg:ADWINbandit}. 
The \AADWIN{}-bandit (adaptive resetting bandit; Algorithm \ref{alg:AADWINbandit}) addresses the two primary issues for the analysis as follows:
First, the updated window $W_2$ after a shrink can include some amount of observations before the changepoint\footnote{We can cope with this problem in the case of single-stream ADWIN (Section \ref{subsec:adwin_abrupt}) since $\mu_W$ eventually converges to $\mu_t$ as $W$ grows. 
However, under bandit feedback, our algorithms focus on optimal arms and draw suboptimal arms less often. Thus, bounding the error of the estimators
for suboptimal arms is hard if the current window $W$ still includes observations before the changepoint.
}. 
To address this issue, \AADWIN{}-bandit resets the entire window for each detection of a changepoint.
Second, it is hard to guarantee a consistent estimation of a single arm.
For example, in a data stream with an abrupt change, if only arm $1$ is selected prior to the change and only arm $2$ is selected after the change, the algorithm is incapable of identifying this change. Although this particular scenario is unlikely to happen, it poses a challenge to the analysis of the algorithm.
To address this issue, \AADWIN{}-bandit introduces the \textit{monitoring arm} that we continue to draw with some frequency.

\AADWIN{}-bandit (Algorithm \ref{alg:AADWINbandit}) works as follows. 
After each observation, if one of the change point detectors identifies a split in the window $W$ into $W_1$ and $W_2$, it resets the entire algorithm (Line \ref{initstart_monitoring}) to start with $W(t+1) = \emptyset$.
Moreover, Algorithm \ref{alg:AADWINbandit} divides the rounds into blocks $l = 1,2,3,\dots$ (Line \ref{adwin_blocks}). Each block consists of $O(2^{l-1})$ subblocks, and each subblock consists of $\wasD N$ rounds. Before the beginning of the last subblock of each $l$, it determines the monitoring arm $i^{(l)}$ (Line \ref{line_monitoring_l}), and then $i^{(l)}$ is drawn once in each $\wasD $ rounds during the next block. Here, monitoring arm is chosen to be $\argmax_i N_i^{(l)}$, where
\begin{equation}\label{ineq_monitordef}
N_i^{(l)} = 
\left\{
\begin{array}{ll}
|\{s \le \wasD N: I(s)=i\}|  & l = 1
\\
|\{(2^{l-1}-1) \wasD N + 1 \le s \le (2^l-2)\wasD N: I(s)=i, \text{Base-Bandit algorithm is called}\}|  & l \ge 2
\end{array}
\right.,
\end{equation}
be the number of draws of arm $i$ in the current block $l$. Note that there require some special treatment for $l=1$ (Line \ref{line_monitoring_one}). 

\subsection{Properties of Algorithm \ref{alg:AADWINbandit}}

As shown in Figure \ref{fig:monitoring}, the duration of two monitoring arms $i^{(l-1)}, i^{(l)}$ are overlapped so that the following properties are satisfied:
\updated{
\begin{mydef}{\rm (Monitoring consistency)}\label{def:monitoring}
\AADWIN{}-bandit with any base-bandit algorithm has the following two properties:
\begin{itemize}
\item Let $t \ge \wasD N + 1$ (i.e., any round in $l \ge 2$). Then, at least one of the arms $\{i^{(l-1)}, i^{(l)}\}$ satisfies the following: This arm was drawn at least $N$ times before round $t$ and will be drawn at least $N$ times within the next $\wasD N$ rounds. %
\item For any block $l=1,2,\dots$, there exists at least one arm that is drawn at least $N$ times for each subblock in $l$. %
\end{itemize}
\end{mydef}
}
The first property in Definition \ref{def:monitoring} is used for the abrupt case, and the second property in Definition \ref{def:monitoring} is used for the gradual case.

Moreover, the following theorem states that \AADWIN{}-bandit is designed so that it does not compromise the performance of the base-bandit algorithm when no reset takes place.
\begin{lem}{\rm (Regret due to monitoring)}\label{lem_monitoringreg}
Assume that we run Algorithm \ref{alg:AADWINbandit} with a base-bandit algorithm with $\wasD  \ge 3$.
Assume that Algorithm \ref{alg:AADWINbandit} has not reset itself until round $S$.
Let 
\[
\ETbase = \{t\in[S]: \text{base-bandit algorithm is evoked}\}
\]
be the set of rounds where the base-bandit algorithm is evoked (i.e., when Line \ref{basebandit_round} in Algorithm \ref{alg:AADWINbandit} is executed). %
\updated{
Then, the following holds:
\begin{equation}
\sum_{t=1}^S \reg(t) 
\le 
4
\left(
\sum_{t \in \ETbase} \reg(t) 
\right)
+ \sum_{t=1}^S 3\eps(t).
\end{equation}
}
\end{lem}
The proof of Lemma \ref{lem_monitoringreg} is found in Appendix \ref{subsec:proof_monitoring}.
\updated{The first term}
is a constant multiplication of the base-bandit regret. 
As a result, it can fully utilize the logarithmic drift-tolerant regret of the base-bandit algorithm, which we formalized in the following Lemma.
\begin{lem}\label{lem_monitoringdtprop}{\rm (Constant multiplication)}
Assume that
we run \AADWIN{}-bandit algorithm with a base-bandit algorithm that has logarithmic drift-tolerant regret (Definition \ref{def:driftrobust}).
Assume that no reset has occurred\footnote{Note that $S$ is a stopping time (random variable).} until round $S$.
Then, there exists a constant $\Cdr = O(1)$ that is independent of $S$, and the regret up to $S$ is bounded as 
\begin{equation}\label{ineq_monitoringdtprop}
\Ex[\Reg(S)] \le 
\Cdr 
\left(
\sum_{i \ne 1} \frac{\log T}{\Delta_i}
+
\Ex\left[
\sum_{t =1}^S \eps(t)
\right]
\right).
\end{equation}
\end{lem}
\begin{proofof}{Proof of Lemma \ref{lem_monitoringdtprop}}
Definition \ref{def:driftrobust} and Lemma \ref{lem_monitoringreg} imply
\[
\Ex\left[
\sum_{t\le S} (\reg(t) - \Cdr \eps(t))_+ 
\right]
\le 
\Ex\left[
\sum_{t\le T} (\reg(t) - \Cdr \eps(t))_+ 
\right]
\le \Cdr \sum_{i \ne 1} \frac{\log T}{\Delta_i},
\]
which implies Eq.~\eqref{ineq_monitoringdtprop}.
\end{proofof}

\subsection{Regret bounds of \AADWIN-bandit algorithms}
\label{subsec:regret}

Having defined the \AADWIN-bandit algorithm and the related properties, we state the main results regarding the performance of these algorithms in stationary abruptly or gradually changing environments.

\begin{thm}{\rm (Regret bound of \AADWIN-bandit, stationary case)}\label{thm:regret_stationary}
Let the environment be stationary. Let the base-bandit algorithm has a logarithmic stationary regret. Then, the regret of Algorithm \ref{alg:AADWINbandit} with $\delta = 1/T^3$ is bounded as follows:
\begin{equation}
\Ex[\Reg(T)] 
\le \Cst \sum_i \frac{\log{T}}{\Delta_i} + O(1). %
\end{equation}
\end{thm}%
Theorem \ref{thm:regret_stationary} states that change-point detectors do not deteriorate the performance of the base-bandit algorithm. The proof directly follows from the following facts. First, adaptive windows do not make a false reset with at least with probability of $1-O(K/T)$. Second, %
the regret of the \AADWIN-bandit algorithm is equal to the base-bandit algorithm when no reset occurs (Lemma \ref{lem_monitoringdtprop}). For completeness, we provide the proof of Theorem \ref{thm:regret_stationary} in Appendix \ref{subsec:regstationary}.
This result has a striking difference from most existing nonstationary bandit algorithms, which have the cost of a forced exploration of $\Omega(\sqrt{T})$. In the following sections, we derive the regret bound in abruptly and gradually changing environments.

We next derive a regret bound for an abruptly changing environment. We first define the global changes that we assume on the nature of the streams (Definition \ref{def:detectable_bandit}) and then state the regret bound.

\komiyama{(made Definition \ref{def:detectable_bandit} simple, most of the detectability assumptions are now on Theorem \ref{thm:regret_abrupt})}
\begin{mydef}{\rm (Globally abrupt changes)}\label{def:detectable_bandit}
Let $\NumChange$ be the number of change points and   
$\ET_c = \{T_1,T_2,\allowbreak\dots,T_{\NumChange}\} = \{t:\exists\,i\in[K]\,\mu_{i,t} \ne \mu_{i,t+1}\}$
be the set of changepoints. Moreover, $(T_{0}, T_{\NumChange+1}) = (0, T)$.
The $m$-th changepoint is a global change with $\Crabrupt$ if 
\begin{equation} 
\max_{i,j \in [K],t \in \ET_c} \frac{|\mu_{j,t} - \mu_{j,t+1}|}{|\mu_{i,t} - \mu_{i,t+1}|} \le \Crabrupt
\end{equation}
is finite.
\end{mydef}
Intuitively, $\Crabrupt$ is the maximum ratio of changes among arms. Assuming it as a constant indicates that any modification in one arm is proportionally mirrored in all the other arms.
\begin{mydef}{\rm (Detectability)}\label{def_detectable_abrupt}
For the $m$-th changepoint, let $\epschg{m} = \min_i |\mu_{i,m} - \mu_{i,m+1}|$.
We let $U(\eps) = (\log (T^3))/(2\eps^2)$.
The $m$-th changepoint is detectable if $T_m - T_{m-1}, T_{m+1} - T_m \ge 48\wasD U(\epschg{m})$.
\end{mydef}
The detectability assumption is essentially the same as those in the literature of active nonstationary bandit algorithms, such as \cite{DBLP:conf/aistats/0013WK019} and \cite{besson2020efficient}.

\begin{thm}{\rm (Regret bound of \AADWIN{}-bandit under abrupt changes)}\label{thm:regret_abrupt} 
Let the base-bandit algorithm have a logarithmic drift-tolerant regret bound (Definition \ref{def:driftrobust}). %
Let the environment have $\NumChange$ detectable globally abrupt changes (Definitions \ref{def:detectable_bandit}, \ref{def_detectable_abrupt}) with $\Crabrupt>0$. 
Then, the regret of the Algorithm \ref{alg:AADWINbandit} 
 is bounded as follows:
\begin{equation}
\Ex[\Reg(T)] 
=
\tilO(\sqrt{MKT}),
\end{equation}
if $\delta = 1/T^3$, $N \ge 16U(\epschg{m}), \wasD N \le (T_{m+1} - T_m)/3$ holds for all $m$.
\end{thm}

\komiyama{(todo update this)}
\begin{sketch} 
The proof sketch is as follows.
\begin{itemize}
    \item With a high probability, the algorithm resets itself within $16 \wasD  U(\Delta_m)$ time steps after $m$-th changepoint there exists a corresponding reset. 
    \item By the sublinear drift-tolerant regret of the algorithm the expected regret between $m$-th and $(m+1)$-th changepoint is $\tilde{O}(\sqrt{K(T_{m+1}-T_{m})})$, where we have used the standard discussion of distribution-independent regret (c.f., p32 in \cite{bubeckthesis}) that utilizes the Cauchy-Schwarz inequality.
    \item The desired bound is yielded by summing the regret over all change points and another application of the Cauchy-Schwarz inequality on $\sum_m (T_m - T_{m-1}) \le T$.
\end{itemize}
The formal proof is found in Appendix \ref{subsec:regret_abrupt}.
\end{sketch}
Theorem \ref{thm:regret_abrupt} implies that the \AADWIN{}-bandit learns the environment when the changes are infrequent (i.e., $M = o(T)$) %
and the changes are detectable. 
This bound matches many existing active algorithms, such as those in \cite{DBLP:conf/aistats/0013WK019,besson2020efficient}. 

We next analyze the performance of \AADWIN-bandit with a gradually changing environment. We employ a technique similar to that in Section \ref{subsec:adwin_gradual} for single-stream ADWIN. We first describe the global changes in a gradually changing environment and then state a regret bound.

\begin{assumption}{\rm (Globally gradual changes)}\label{as:detectable}
A set of $K$ streams has globally gradual changes if a constant $\Cdetect \in (0,1]$ exists such that 
\begin{equation}
|\mu_{i,t} - \mu_{i,s}| 
\ge 
\Cdetect |\mu_{j,t} - \mu_{j,s}| 
\end{equation}
holds for any two arms $i,j \in [K]$ and any two time steps $t,s \in [T]$.
\end{assumption}
Intuitively, Assumption \ref{as:detectable} states that all arms are drifting in a coordinated manner. The drift on the mean of an arm is proportional to the drift on the means of other arms.

\begin{thm}{\rm (Regret bound of the \AADWIN{}-bandit under gradual changes)}\label{thm:regret_gradual}  %
Let the base-bandit algorithm have a logarithmic drift-tolerant regret (Definition \ref{def:driftrobust}) and have monitoring consistency (Definition \ref{def:monitoring}). Let the environment be gradual with change speed $b$ and $d$ be such that $b = T^{-d}$. Let the changes be global (Assumption \ref{as:detectable}).
Then the regret of Algorithm \ref{alg:AADWINbandit} with $\delta = 1/T^3, N=\tilde{\Theta}((b\wasD )^{-2/3})$ is bounded as:
\begin{equation}
\Ex[\Reg(T)] 
 = \tilO\left(\sqrt{K} T^{1-d/3}\right).
\end{equation}
\end{thm}
\begin{sketch}%
The proof sketch is as follows.
\begin{itemize}
    \item Using Lemma \ref{lem:errorshrink_bandit}, the current window grows until $|W| = O(b^{-2/3})$ with a high probability, which implies that the number of resets (= detection times) $\Mbreak$ is at most $\tilde{O}(T/b^{2/3}) = \tilde{O}(T^{1-2d/3})$.
    \item With a high probability, $|\mu_t - \mu_W| = \tilde{O}(b \wasD  N) + \tilde{O}\left( \sqrt{\log{T}/N} \right)$ always holds (Lemma \ref{lem:gradualbound_bandit}). This term is minimized when $N=\tilde{\Theta}((b\wasD )^{-2/3})$ and $|\mu_t - \mu_W| = \tilde{O}((bK)^{1/3})$. 
    \item Similarly to abrupt case, the regret between the $m$-th and $(m+1)$-th resets is bounded in expectation as $\tilde{O}(\sqrt{K(\Td{m+1}-\Td{m})})$ plus the drift term of $\sum_t \tilde{O}((bK)^{1/3})$.
    \item By using the Cauchy-Schwarz inequality on $\sum_{m=1}^{\Mbreak} \Td{m+1}-\Td{m} = T$ and $\Mbreak = \tilde{O}(T^{1-2d/3})$, we have $\sum_{m=1}^{\Mbreak} \tilde{O}(\sqrt{K(\Td{m+1}-\Td{m})}) = \tilde{O}(\sqrt{\Mbreak KT}) = \tilde{O}(\sqrt{K}T^{1-d/3})$. Moreover, $\sum_t \tilde{O}((bK)^{1/3}) = T \tilde{O}((bK)^{1/3}) = \tilde{O}(K^{1/3}T^{1-d/3})$.

\end{itemize}
The formal proof is in Appendix \ref{subsec:regret_gradual}.
\end{sketch}
Theorem \ref{thm:regret_gradual} states that \AADWIN{}-bandit learns the environment when the change is sufficiently slow (i.e., $b = o(1)$). %
Unlike most nonstationary bandit algorithms, our algorithm design does not require prior knowledge of $b$.
\begin{myrem}{\rm (Optimality of the regret bound)
The rate matches the $\tilTheta(V^{1/3} T^{2/3})$ lower bound of  \cite{besbes2014optimal} up to a logarithmic order. Under a gradually changing environment, the total variation is $V:= \sum_t |\mu_{i,t} - \mu_{i,t+1}| = bT = O(T^{1-d})$; thus, $V^{1/3}T^{2/3} = O(T^{1-d/3})$. A recent work by \cite{krishnamurthy2021} shows the optimality of this rate for gradually changing streams.\footnote{The result from \cite{DBLP:journals/jair/TrovoRG20} has a better rate of $\tilO(T^{1-d})$. However, they require an additional assumption (Assumption 2, therein) that essentially states that the two arms are not too close to distinguish for a long time.} 
Note that, unlike \cite{besbes2014optimal,krishnamurthy2021}, our guarantee in the nonstationary setting is limited to the case of global changes.
}\end{myrem}

In summary, \AADWIN{}-bandit adapts to both the abruptly and gradually changing environments 
\myred{under mild knowledge on the speed of changes}
if the base-bandit satisfies has a logarithmic drift-tolerant regret bound.  We demonstrate that TS has this bound and consider that it is not very difficult to derive the same bound for KL-UCB.

\section{Experiments}
\label{sec:experiments}

This section reports the empirical performance of our proposed algorithms in simulations. We release the source code to reproduce our experiments on GitHub\footnote{\url{https://github.com/edouardfouche/G-NS-MAB}}.

\subsection{Environments}

The rewards from arm $i$ are drawn from a Bernouilli distribution with parameter $\mu_i$ that we determine for each arm. We define the following synthetic environments: 

\begin{itemize}
    \item \textbf{Stationary}: We define a stationary environment (i.e., with no change) with $100$ arms where the mean $\mu_i = (101-i)/100$ never changes over time, i.e., the rewards of arms are evenly distributed between $0$ and $1$.
    \item \textbf{Gradual}: We define a gradually changing environment with $100$ arms with $\mu_{i,t} = \frac{(T-t+1)}{T}\mu_{i,1} + \frac{t-1}{T}(1-\mu_{i,1})$, where $\mu_{i,1}=(101-i)/100$. In this environment, the mean of each arm is the same as the stationary environment at time $t=1$, but evolves gradually toward $1-\mu_{i,1}$ as $t$ increases up to $T$.
    \item \textbf{Abrupt}: We define an abruptly changing environment with $100$ arms and $\mu_{i,1}=(101-i)/100$ at $t=1$. Unlike the stationary environment, $\mu_{i,t} = 1 - \mu_{i,1}$ between round $t'=T/3$ and round $t''= 2T/3$. In other words, the rewards of every arms change abruptly two times at $t=T/3,2T/3$.  %
    \item \textbf{Abrupt local}: We define an environment with $100$ arms in which only a portion of the arms change abruptly, i.e., the global change assumption does not hold. As in the other environment, $\mu_{i,1}=(101-i)/100$ at $t=1$, but for $i \in [1,10]$, $\mu_{i,t} = 0.5$ between round $t'=T/3$ and round $t''= 2T/3$. In other words, only the top-10 arms change abruptly two times, and the other arms are stationary. 
\end{itemize}

The stationary, abrupt, and gradual settings are in line with the assumptions from our algorithms (\ADWIN{}-TS, \AADWIN{}-TS, \AADWIN{}-KL-UCB) as the changes are global. The abrupt local setting violates these assumptions because only a subset of the arms change.

\begin{figure}[t]
\begin{center}
    \includegraphics[width=1.\linewidth]{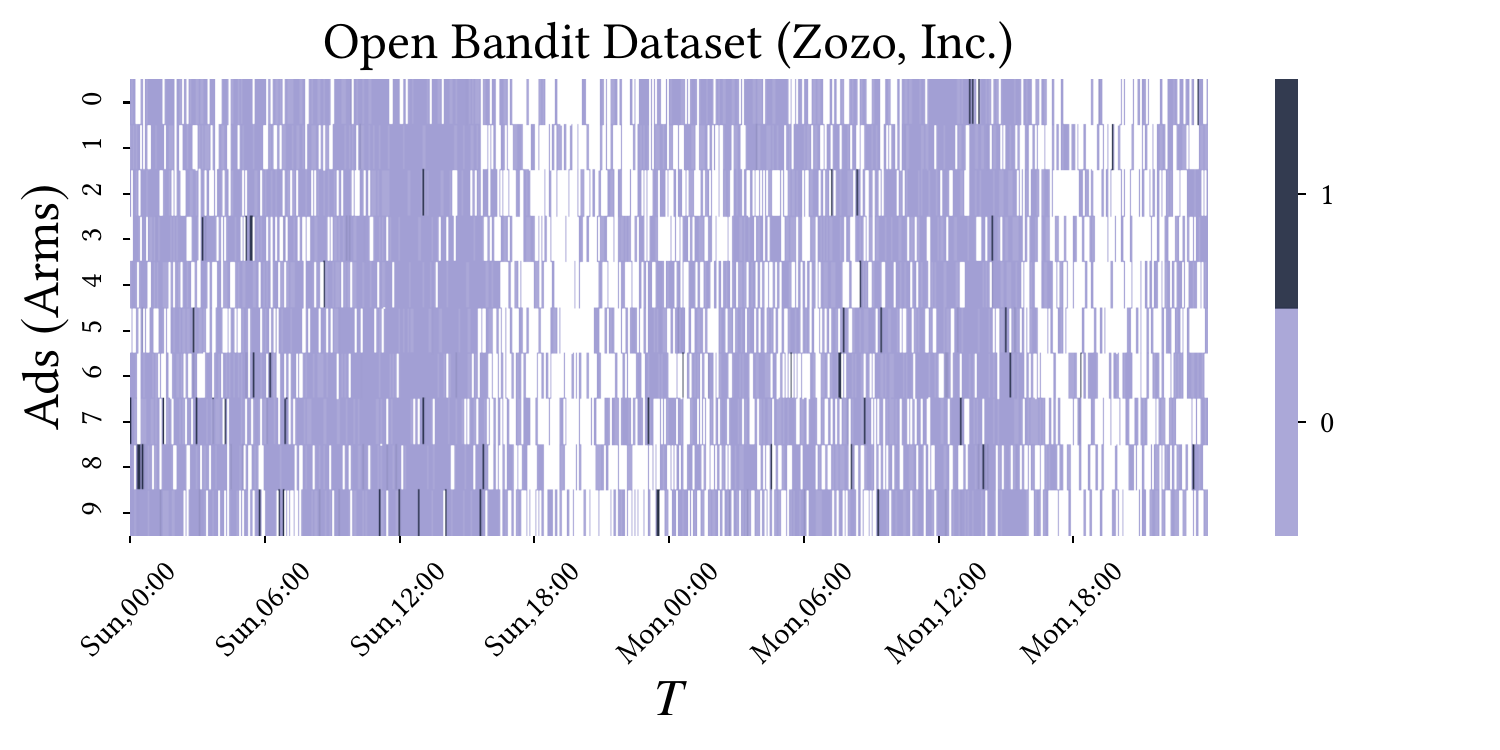}
\end{center}
	\caption{Rewards in the Zozo dataset. We chose ten ads  (= arms, $y$-axis) among $80$ ads. Dark blue and light blue indicate the reward $1$ and $0$, respectively. White indicates that the corresponding ad is not shown at that time step (no feedback). The ratio of reward $1$ to reward $0$ is between $0.5\%$ to $3\%$. We only show the first two days, but the dataset consists of one week of data.}
	\label{fig:zozo_heatmap}
\end{figure}

In addition, we consider the following real-world scenarios: 

\begin{itemize} %
    
    \item \textbf{Bioliq}: This dataset was released by \cite{DBLP:conf/kdd/FoucheKB19}. The rewards are generated from a stream of measurements between 20 sensors in power plant for a duration of 1 week. The authors considered the task of monitoring high-correlation between those sensors and wanted to use bandit algorithms for efficient monitoring. Each pair of sensors is seen as an arm (i.e., there are $20*19/2 = 190$ arms) and the reward is $1$ in case the correlation is greater than some threshold over the last $1000$ measurements, otherwise $0$. Figure \ref{fig:bioliq_heatmap} is the reward distribution for this environment, as we can see, arms tend to change periodically together. 
    
    \item \textbf{Zozo} \citep{saito2020large} is a real-world environment where the rewards are generated from an ad recommender on an e-commerce website. We use the data generated by the uniform recommender from the duration of a week and adopted an unbiased offline evaluator \citep{DBLP:conf/wsdm/LiCLW11} for our experiment\footnote{See also \url{https://github.com/st-tech/zr-obp}}. There are $80$ different ads. Due to the sparseness of the rewards, we picked $10$ arms. Since only a few of such ads are presented at each round to each user, we concatenated together all the ads presented within a second and assigned a reward of $1$ to ads on which at least one user clicked. We assign a reward of $0$ to ads that were presented, but no user clicked on them. We skip the round whenever the selected was not presented at all. Figure \ref{fig:zozo_heatmap} shows the reward distribution of the dataset.

\end{itemize}
All the results are averaged over $100$ runs. Note that these datasets match with our motivational examples (Example \ref{expl:ad} and \ref{expl:bioliq}). 

\subsection{Bandit algorithms}
\label{subsec:exp_setting}

We compared the following bandit algorithms:  
\begin{itemize}
    \item \textbf{ADS-TS} is the \ADWIN{}-bandit algorithm (Algorithm \ref{alg:ADWINbandit}) with TS (Algorithm \ref{alg:mpts}) as a base-bandit algorithm.

    \item \textbf{ADR-TS} and \textbf{ADR-KL-UCB} are the \AADWIN{}-bandit algorithm (Algorithm \ref{alg:AADWINbandit}) with TS (Algorithm \ref{alg:mpts}) and KL-UCB (Algorithm \ref{alg:mpklucb}) as a base-bandit algorithm, respectively.
    \item Passive algorithms: \textbf{RExp3} \citep{besbes2014optimal} is an adversarial bandit algorithm. Discounted UCB (\textbf{D-UCB}) \citep{Garivier2008} and Sliding Window TS (\textbf{SW-TS}) \citep{DBLP:journals/jair/TrovoRG20}, \textbf{SW-UCB\#} \citep{DBLP:conf/amcc/WeiS18} (comes in two variants: for Abrupt (A) and Gradual (G) changes) are bandit algorithms with finite memory.
    \item Active algorithms: \textbf{GLR-klUCB} \citep{besson2020efficient} is an likelihood-ratio based change detector. It comes in two variants: for local and global changes, and we show the results of the latter.
    \textbf{M-UCB} \citep{DBLP:conf/aistats/0013WK019} is a bandit algorithm using another change detector as ours. These algorithms involve forced exploration (uniform exploration of size $O(\sqrt{T})$) unlike our algorithms. \textbf{UCBL-CPD} and \textbf{ImpCPD} \citep{Mukherjee2019} are  algorithms for global abrupt changes that are also free from forced exploration.
\end{itemize}
The original implementation of ADWIN requires $O(|W|K)$ times at each round, which can be large in the case of stationary environments.
\cite{bifet2007learning} introduced a more computationally efficient version of ADWIN, named ADWIN2, which only checks a logarithmic number of such splits. We adopt ADWIN2 in our simulations.

\begin{figure}
\begin{center}
	\includegraphics[width=\linewidth]{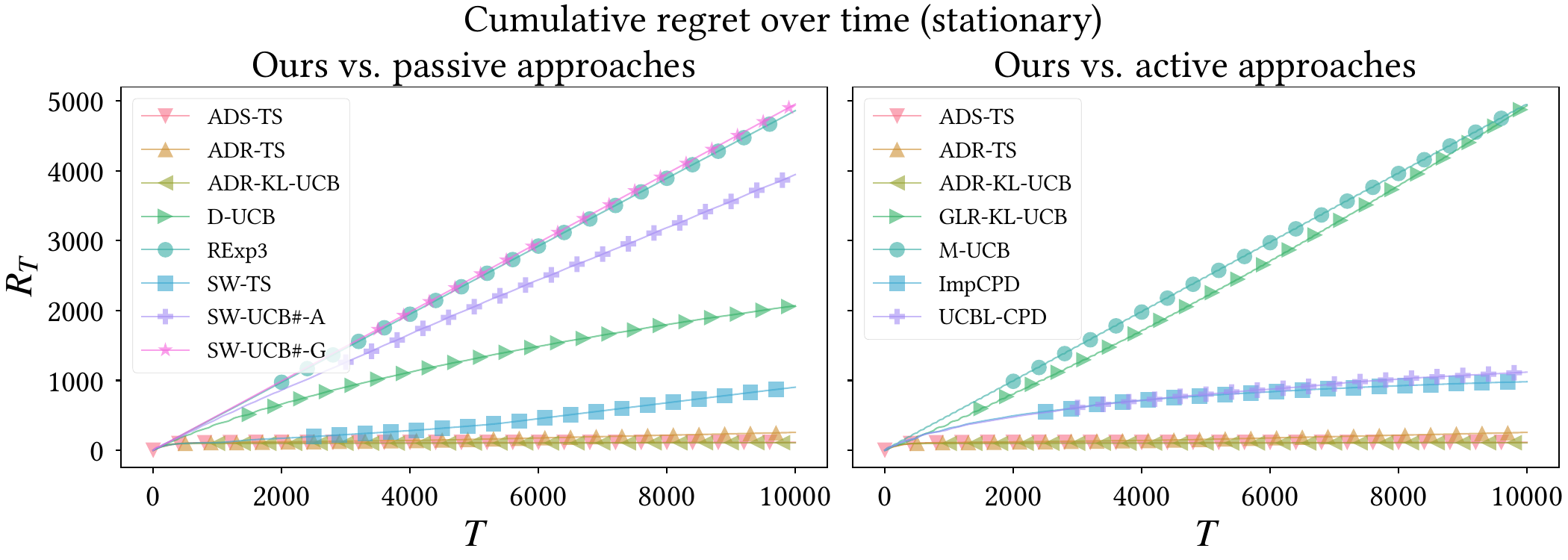}
	\includegraphics[width=\linewidth]{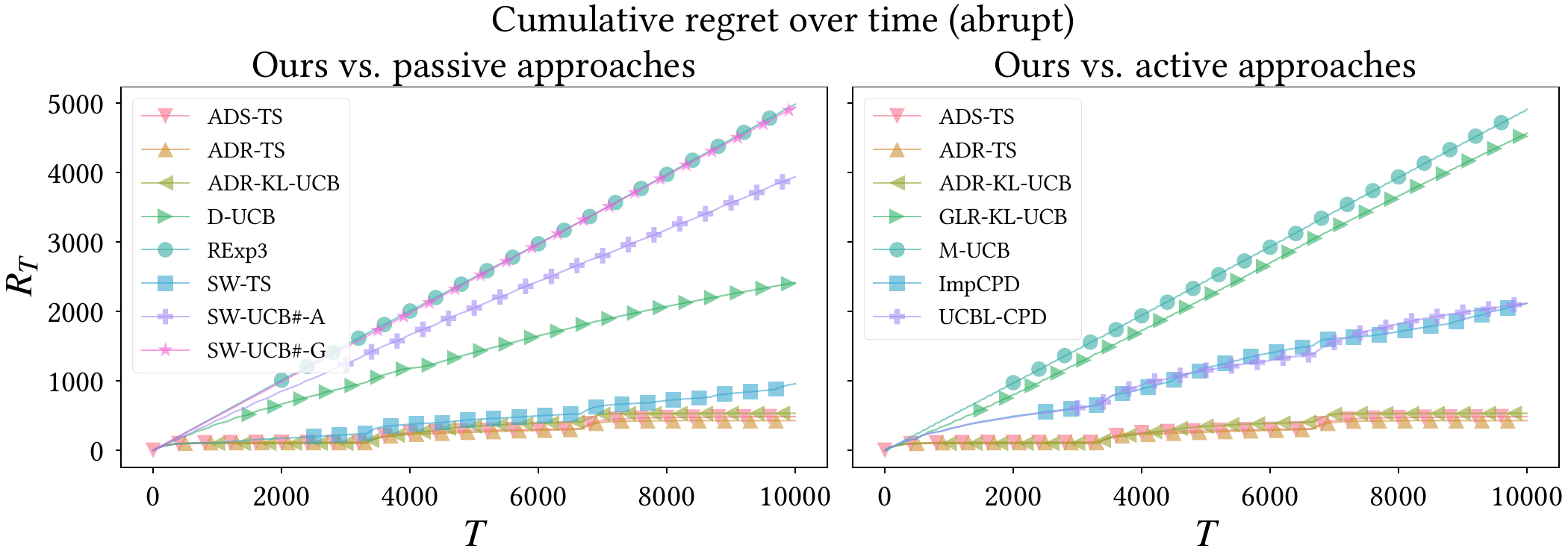}
    \includegraphics[width=\linewidth]{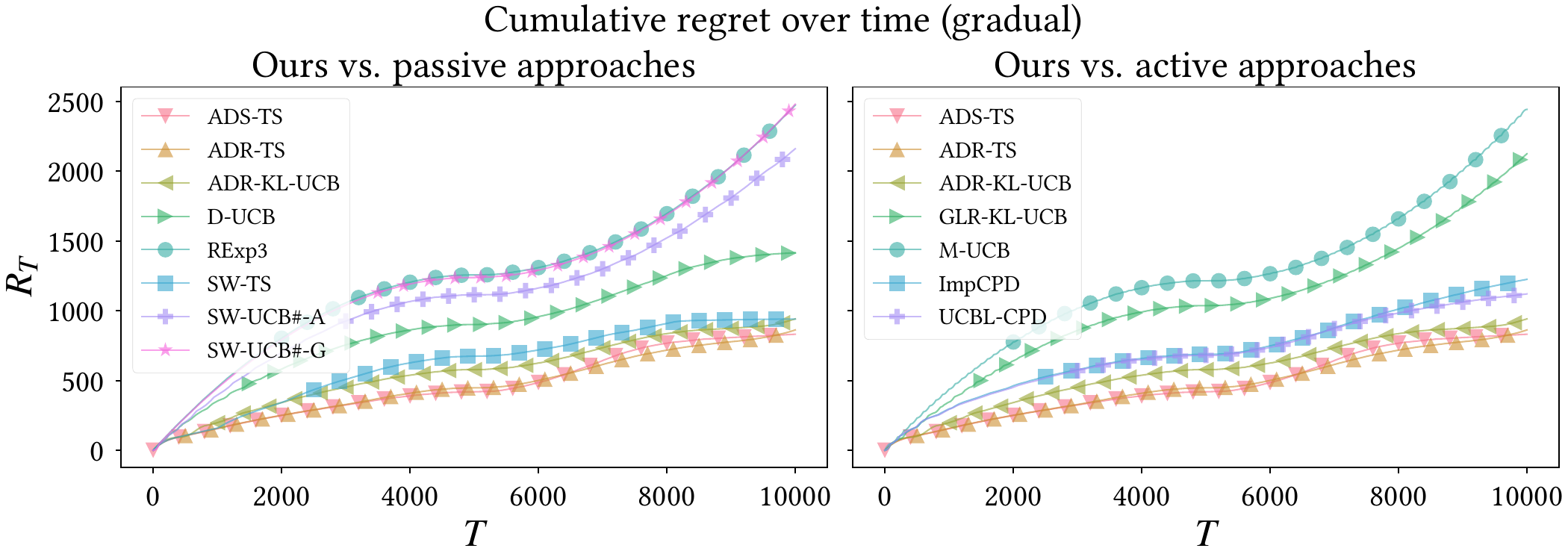}
	\caption{Regret of algorithms in stationary, abrupt, and gradual  environments. Smaller regret ($R_T$: y-axis) indicates better performance.} 
	\label{fig:synth}
\end{center}
\end{figure}

\begin{figure}
\begin{center}
	\includegraphics[width=0.95\linewidth]{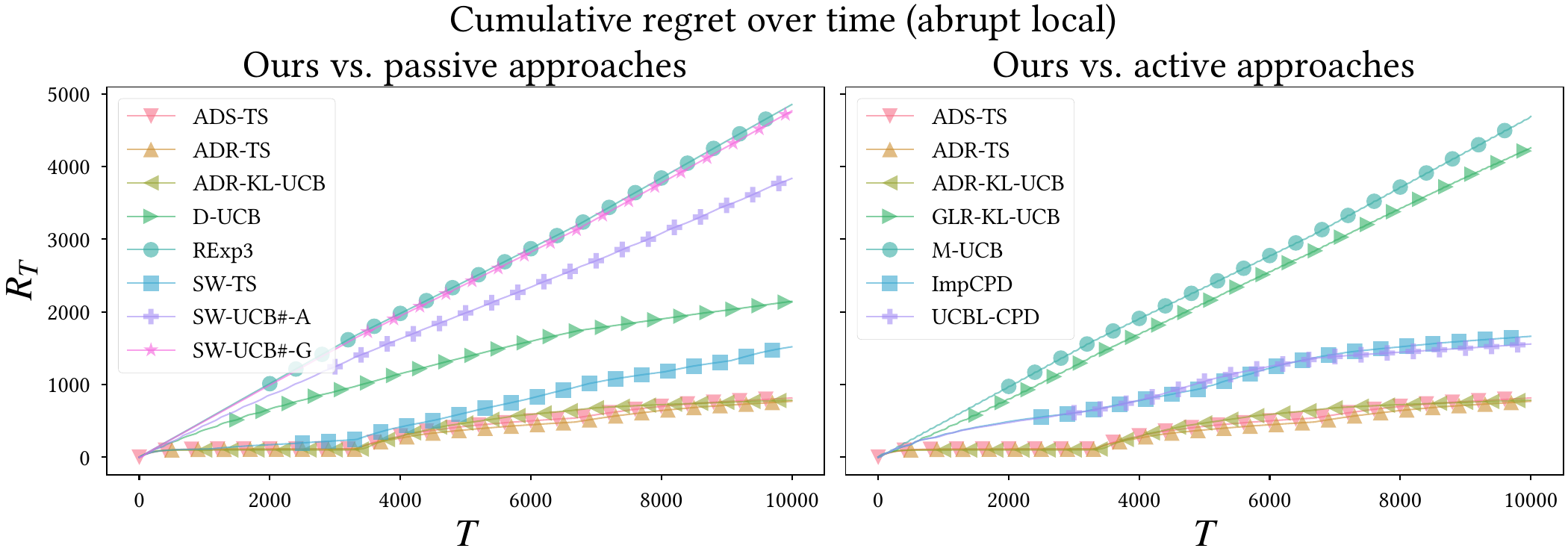}
	\caption{Regret of algorithms in abrupt local environments. Smaller regret ($R_T$: y-axis) indicates better performance. \AADWIN{}-TS has linear regret after two local changes. However, it still outperforms algorithms that require uniform exploration.}
	\label{fig:abruptlocal}
\end{center}
\end{figure}

\begin{figure}
\begin{center}
	\includegraphics[width=\linewidth]{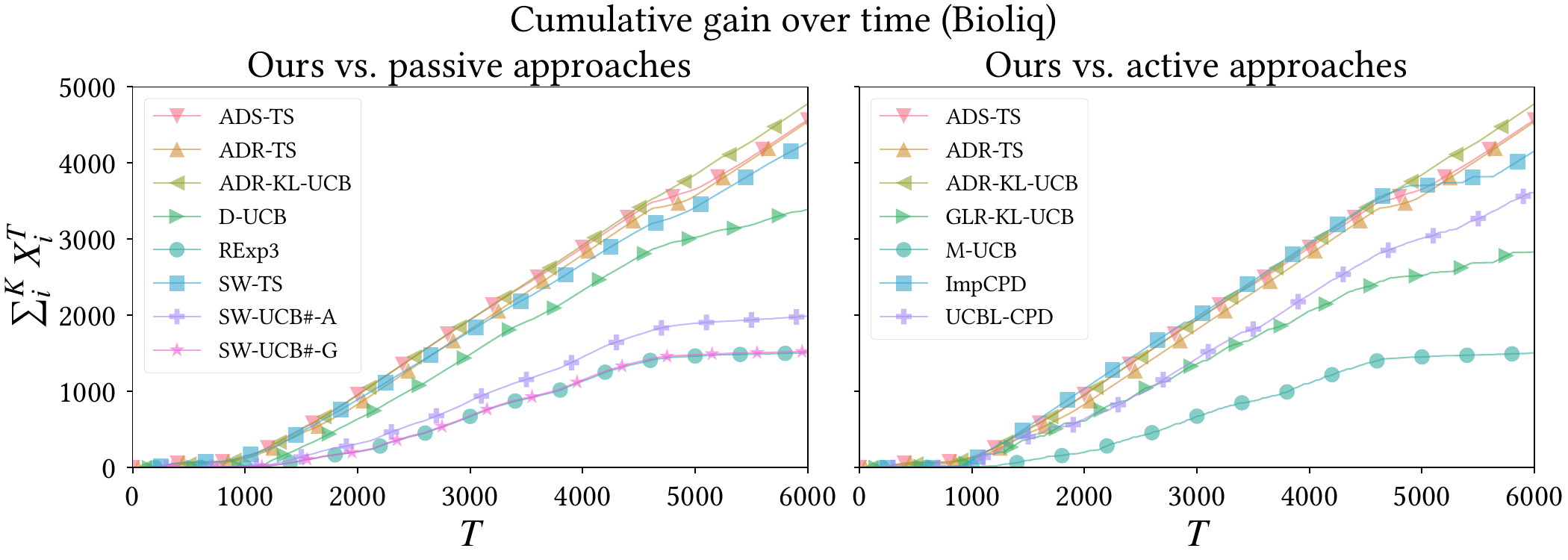}
    \includegraphics[width=\linewidth]{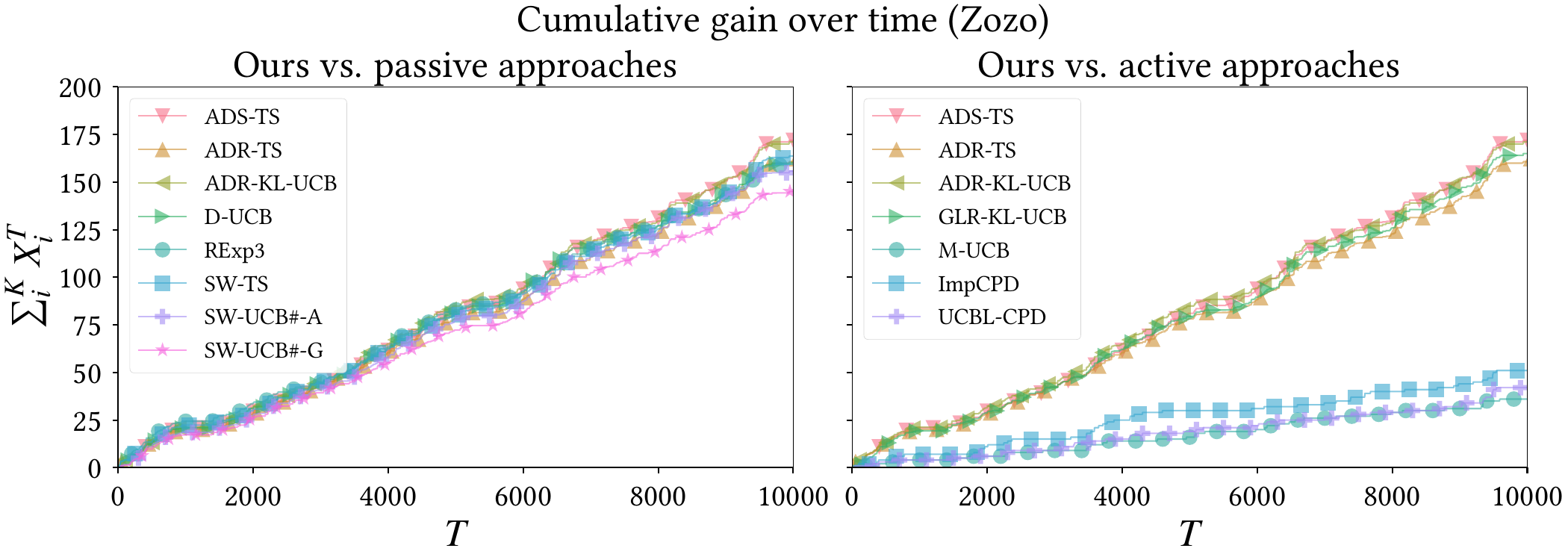}
	\caption{Performance of algorithms based on two real-world datasets (Bioliq and Zozo). Unlike synthetic environments, we do not have ground truth of $\mu_{i,t}$ and thus we report cumulative reward. Larger cumulative reward (y-axis) indicates better performance. M-UCB tends to have $\gamma \ge 1$ in our setting, which results in a uniform exploration.}
	\label{fig:bioliqzozo}
\end{center}
\end{figure}

\subsection{Experimental results}
\label{subsec:exp_results}

First, Figure \ref{fig:synth} compares \AADWIN{}-bandit and existing algorithms in the stationary, abrupt, and gradual environments. In the stationary environment, \AADWIN{}-TS has very low regret, which is consistent with the fact that its regret is logarithmic. In the two nonstationary environments with global changes, \AADWIN{}-TS outperforms the other algorithms. 
Second, Figure \ref{fig:abruptlocal} compares these algorithms in the abrupt local environment. Although the \AADWIN{}-TS has linear regret as it sometimes do not detect these two local changes, it still outperforms the other algorithms. 
Finally, Figure \ref{fig:bioliqzozo} compares these algorithms against the Bioliq and Zozo datasets. The proposed algorithm consistently outperforms existing nonstationary bandits algorithms. 

We provide in Section \ref{sec:additionalexperiment} additional simulations on the effect of resetting (compared with the shrinking of the original ADWIN), the sensitivity \AADWIN{}-TS to the value of parameter $\delta$, and results with $L>1$. We find that \AADWIN{}-TS is not very sensitive to the choice of parameter $\delta$, as long as it is moderately small.

\section{Conclusions}
\label{sec:conclusions}

We have expanded stationary bandit algorithms to nonstationary settings without requiring forced exploration. 
To this aim, we first analyzed the theoretical property of adaptive windows in a single-stream setting (Section \ref{sec:adwin}).
After that, we combined bandit algorithms (Section \ref{sec:MP-MAP}) with adaptive windowing by introducing \AADWIN{}-bandit.
Unlike existing algorithms, \AADWIN{}-bandit does not act on the base-bandit algorithm unless change points are detected, and thus, it does not compromise the performance in a stationary environment. 
We have demonstrated its ability to detect global changes in Section \ref{sec:NS-MP-MAP-analysis} and tested it in simulated and real-world settings through experiments in Section \ref{sec:experiments}.

\newpage

\appendix

\clearpage

\bibliography{sample-base}

\clearpage

\appendix

\section{Hyperparameters in Experiments}

The hyperparameters of the algorithms are shown in Table \ref{tbl:hyperparams}. %
\begin{table}[]
\small
\caption{
List of tested hyperparameters of the algorithms. Bold letters indicate the ones reported in this paper. We set the confidence level of UCBL-CPD and ImpCPD to be the same as ADR-TS.} %
\label{tbl:hyperparams}
\begin{center}
\renewcommand{\arraystretch}{1.4}
\begin{tabular}{l|l}
Algorithm(s) & Hyperparameters  \\ \hline%
\ADWIN{}-TS and \AADWIN{}-TS, \AADWIN{}-KL-UCB & $\delta=10^{-1}, 10^{-2}, \bm{10^{-3}}, 10^{-4}, 10^{-5}, 10^{-6}, 10^{-8}, 10^{-12}, 10^{-15}$  \\   
RExp3 & $\Delta_T = 100, 500, \bm{1000}, 5000$  \\  
SW-TS & $W = 100, 500, \bm{1000}, 5000$  \\  
D-UCB & $\gamma = 0.7, 0.8, \bm{0.9}, 0.99, 0.999$  \\  
SW-UCB\#-A & $\nu = \bm{0.1}, 0.2$ and $\lambda = 12.3$  \\  
SW-UCB\#-G & $\kappa = \bm{0.1}, 0.2$ and $\lambda = 4.3$   \\  
GLR-klUCB & $\alpha=\bm{\sqrt{k A \log(T)/T}}$ and $\delta = \bm{1/\sqrt{T}}$  \\ 
M-UCB & $w = \bm{1000}, 5000$ and $M = \textbf{10}, 100$\\  %
UCBL-CPD & $\bm{\delta = 10^{-3}}$\\ 
ImpCPD & $\bm{\gamma = 0.05, \psi\eps_m^2=10^{3}}$
\end{tabular}
\end{center}
\end{table}

\section{Additional Experiments}
\label{sec:additionalexperiment}

This section reports the results of additional experiments.

\begin{figure}
\begin{center}
	\includegraphics[width=\linewidth]{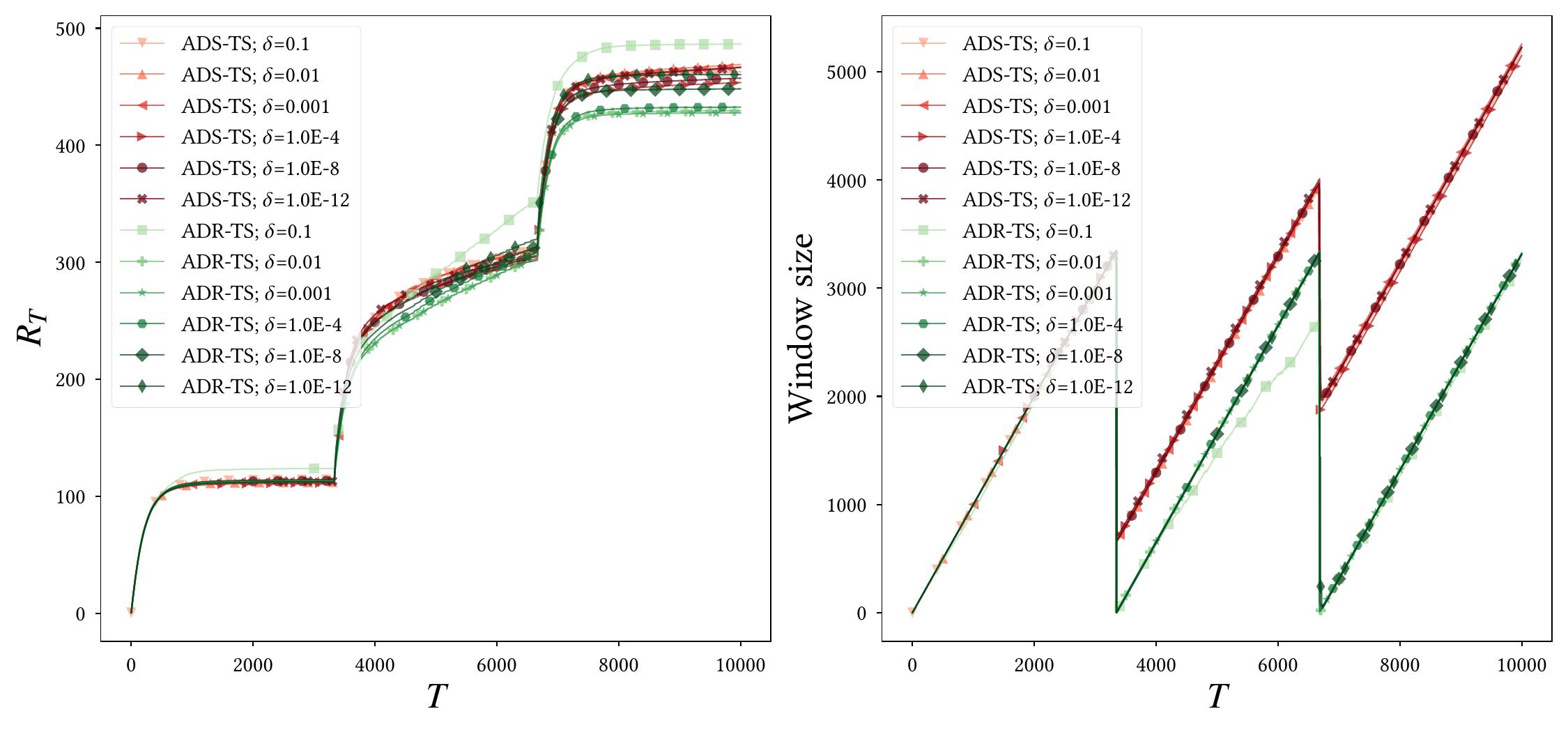}
	\includegraphics[width=\linewidth]{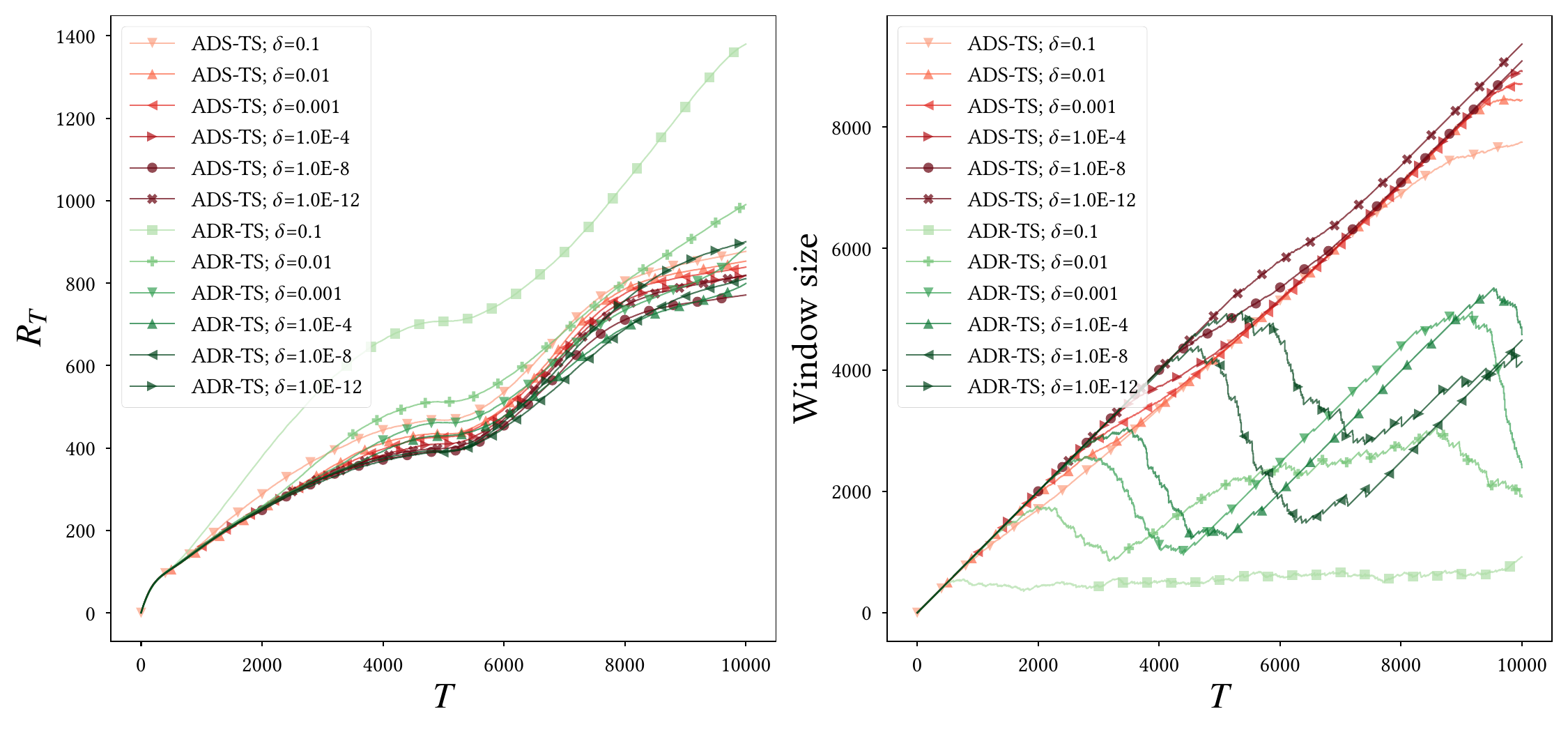}
	\caption{Comparison of \AADWIN{}-TS and \ADWIN{}-TS in synthetic environments. Left: Smaller regret  ($R_T$: y-axis) indicates better performance. Right: Size of the window over time. The performance gap between \AADWIN{}-TS and \ADWIN{}-TS is very small. The effect of hyperparamater $\delta$ in change point detection is not very large, except for a very large or very small value of $\delta$.}
	\label{fig:abrupt_adrads}
\end{center}
\end{figure}

\clearpage

\section{Lemmas}\label{sec:lemmas}

This section describes the lemmas that are used in the proofs of this paper.

The Hoeffding inequality, which is one of the most well-known concentration inequality, provides a high-probability bound of the sum of bounded independent random variables. 
\begin{lem}[Azuma-Hoeffding inequality]\label{lem:hoeffding}
Let $x_1,x_2,\dots,x_n$ be martingale random variables in $[0,1]$ with their conditional mean $\mu_m = \Ex[x_m|x_1,x_2,\dots,x_{m-1}]$. Let $\bar{x} = (1/n)\sum_{t=1}^n x_t$ and $\bar{\mu} = (1/n)\sum_{t=1}^n \mu_t$. Then, 
\begin{align}
\Pr\left[\bar{x} - \bar{\mu} \ge \sqrt{\frac{\log(1/\delta)}{2 n}} \right] &\le \delta, \\
\Pr\left[\bar{x} - \bar{\mu} \le -\sqrt{\frac{\log(1/\delta)}{2 n}} \right] &\le \delta.
\end{align}
Moreover, taking a union bound over the two inequalities yields
\begin{equation}
\Pr\left[|\bar{x} - \bar{\mu}| \ge \sqrt{\frac{\log(1/\delta)}{2 n}} \right] \le 2\delta.
\end{equation}
\end{lem}

\section{Proofs of ADWIN}

We denote $\EA, \EB = \EA \cap \EB$ for two events $\EA, \EB$.

\subsection{Proof of Theorem \ref{thm:adwinabrupt}}
\label{subsec:proof_adwinabrupt}

The overall idea here is as follows. 
With sufficiently long time after a changepoint, we can expect that ADWIN shrinks the window. However, there might be some time steps left in the current window $W(t)$ even if a shrink occurs.
Still, we can show that 
\begin{align}\label{ineq:adwinabrupt_roughgoal}
|\hatmu_{W(t)} - \mu_t| 
&\le 
|\mu_{W(t)} - \mu_t| 
+
|\hatmu_{W(t)} - \mu_{W(t)}| \\
&\le 
\tilO\left(
\sqrt{\frac{1}{c(t)}}
\right)
+
\tilO\left(
\sqrt{\frac{1}{c(t)}}
\right) 
\end{align}
where $c(t)$ be the number of time steps after the changepoint.\footnote{A formal definition of $c(t)$ is given in Eq.~\eqref{ineq:ctdef}.} Events $\EC$ and $\ED$ in the following corresponds to the bounds of $|\mu_{W(t)} - \mu_t|$ and $|\hatmu_{W(t)} - \mu_{W(t)}|$, respectively.

\begin{proofof}{Proof of Theorem \ref{thm:adwinabrupt}}

By Eq.~\eqref{ineq:allwindowsbound}, event
\begin{equation} %
\EB = \bigcap_{W' \in \EW} \left\{ |\mu_{W'} - \hatmu_{W'}| \le \sqrt{\frac{\log (T^3)}{2 |{W'}|}}\right\}
\end{equation}
holds with probability at least $1-2/T$.

For each time step $t \in [T]$, let  
\begin{equation}\label{ineq:ctdef}
c(t) = t - \max_{s < t, s \in \ET_c} s.
\end{equation}
Namely, $c(t)$ is the number of the time steps
\myred{since} the latest changepoint.
Moreover, for $c(t) > 2$, let %
\begin{equation}
\Delta(c) = 
4\sqrt{ 
\frac{\log(T^3)}{c - 2}
}
\end{equation}
and $\Delta(c) = 1$ for $c \le 2$.
Let 
\begin{equation}\label{ineq:abruptdeltabound}
\EC = \bigcap_t \{ |\mu_{W(t)} - \mu_t| < \Delta(c(t)) \}.
\end{equation}
In the following we show $\EC$ holds under $\EB$. 
If $|W(t)|+1 \le c(t)$, then all the time steps in $W(t)$ \myred{are} after the last change point, which implies $\mu_{W(t)} = \mu_t$ and thus $\EC$. 
Otherwise, let 
\begin{align}
W_2 &= \{t-c(t)/2+2,t-c(t)/2+3,\dots,t\} \text{\ \ \ \ (i.e. the last $c(t)/2-1$ time steps)} \\
W_1 &= W(t-1) \setminus W_2
\end{align}
Clearly $|W_1|,|W_2| \ge c(t)/2-1$ and $\mu_{W_2} = \mu_t$.
ADWIN at time step $t-1$ shrinks the window until
\begin{equation}\label{ineq_adv_shrinktill}
|\hatmu_{W_1}| - |\hatmu_{W_2}|
\le
\sqrt{\frac{\log (T^3)}{2|W_1|}} + \sqrt{\frac{\log (T^3)}{2|W_2|}} 
\end{equation}
holds. 
Let $n_1, n_2 = |W_1|, |W_2|$. 
\begin{align}
\sqrt{\frac{\log (T^3)}{2|W_1|}} + \sqrt{\frac{\log (T^3)}{2|W_2|}} 
&\ge 
|\hatmu_{W_1}| - |\hatmu_{W_2}| \text{\ \ \ \ \ (by Eq.~\eqref{ineq_adv_shrinktill})} \\
&\ge 
|\mu_{W_1}| - |\mu_{W_2}| - \sqrt{\frac{\log (T^3)}{2 |W_1|}} - \sqrt{\frac{\log (T^3)}{2 |W_2|}} \text{\ \ \ \ \ (by $\EB$)}\\
&=
|\mu_{W_1}| - |\mu_t| - \sqrt{\frac{\log (T^3)}{2 |W_1|}} - \sqrt{\frac{\log (T^3)}{2 |W_2|}} \text{\ \ \ \ \ (by $|W_2| + 1\le c(t)$)} \label{ineq:ineq_adv_shrinktill_two}
\end{align}
and thus
\begin{align}
|\mu_{W(t)} - \mu_t| 
& \le |\mu_{W_1} - \mu_t| \\
&\le 2 \sqrt{\frac{\log (T^3)}{2|W_1|}} + 2\sqrt{\frac{\log (T^3)}{2|W_2|}} 
 \text{\ \ \ \ \ (by Eq.~\eqref{ineq:ineq_adv_shrinktill_two})} \\
&\le 4\sqrt{\frac{\log (T^3)}{c - 2}}. \text{\ \ \ \ \ (by $|W_1|,|W_2| \ge (c-2)/2$)} 
\end{align}
In summary, $\EB$ implies $\EC$.

Note that under event $\EB$, ADWIN never makes a false shrink (i.e., a shrink when $\mu_{W_1} = \mu_{W_2}$). 
This implies that between two changepoints a shrink that makes $|W(t)|+1 < c(t)$ occurs at most once, which leads to the fact that event
\begin{equation}\label{ineq:new_wlower}
\ED = \bigcap_{n \in \Natural} 
\left\{
\sum_t \Ind\{|W(t)| = n \cap |W(t)|+1 < c(t))\} \le 2M
\right\}.
\end{equation}

By using this, the total error is bounded as
\begin{align}
\lefteqn{
\Ex[\Err(T)]
}\\
&\le 
T \Pr[\EB^c]
+
\Ex[\Err(T)\Ind[\EB,\EC,\ED]] 
\text{\ \ \ \ \ ($\EB$ implies $\EC$ and $\ED$)}\\
&\le 
2
+
\Ex[\Err(T)\Ind[\EB,\EC,\ED]] \\
&=
2 +
\Ex\left[\sum_t |\mu_t - \hatmu_W| \Ind[\EB,\EC,\ED]\right]\\
&\le 
2 +
\Ex\left[\sum_t \left(|\mu_t - \mu_W| + |\hatmu_W - \mu_W|\right) \Ind[\EB,\EC,\ED]\right]\\
&\le 
2 +
\sum_t \Delta(c(t)) 
+ \Ex\left[\sum_t |\hatmu_W - \mu_W| \Ind[\EB,\EC,\ED]\right]
\text{\ \ \ \ (by $\EC$)}
\\
&\le 
2 +
\sum_t \Delta(c(t)) 
+ O(M) + \sum_t \sqrt{\frac{\log (T^3)}{2 c(t)}} \\
&\ \ \ \ \ + \Ex\left[ 
\sum_n \sum_t \Ind[|W(t)| = n, |W(t)| + 1 < c(t)] \sqrt{\frac{\log (T^3)}{2 n}} \Ind[\ED]
\right]
\text{\ \ \ \ (by $\EB$)} \\
&\le O(M) + 
\sum_t O\left( \sqrt{\frac{\log (T)}{c(t)}} \right)
+ \Ex\left[ 
\sum_n \sum_t \Ind[|W(t)| = n, |W(t)| + 1 < c(t)] \sqrt{\frac{\log (T^3)}{2 n}} \Ind[\ED]
\right].
\end{align}
Here 
\begin{align}
\sqrt{\frac{\log(T)}{c(t)}}
&\le
\sqrt{T} \times \sqrt{\sum_t \left(\frac{\log(T)}{c(t)}\right)^2} \text{\ \ \ (by Cauchy-Schwarz inequality)} \\
&= 
\tilO\left(\sqrt{M T}\right). \text{\ \ \ ($c(t) = n$ holds at most $M+1$ times for each $n$)}
\end{align}
Another application of the Cauchy-Schwarz inequality, combined with $\ED$, yields
\begin{equation}
\sum_n \sum_t \Ind[|W(t)| = n, |W(t)|\le c(t)] \sqrt{\frac{\log (T^3)}{2 n}} \Ind[\ED] 
=
\tilO\left(\sqrt{M T}\right),
\end{equation}
which completes the proof.
\end{proofof}

We next discuss the error bound under gradual drift. The following Lemmas \ref{lem:gradualbound} and \ref{lem:errorshrink} characterize the accuracy of estimator $\hatmu_W$ under gradual drift. These lemmas are used in the proof of Lemma Theorem \ref{thm:adwingradual}.

\subsection{Lemma \ref{lem:gradualbound}}

\begin{lem}\label{lem:gradualbound} %
Let the stream be gradual with its speed of the change $b$.
Let the parameter of ADWIN be $\delta = 1/T^3$.
Then, with probability at least $1-2/T$, 
\begin{equation}\label{ineq:gradualbound}
 |\mu_s - \mu_{s'}| \le 3 b N + \tilO\left(\sqrt{\frac{1}{N}}\right)
\end{equation}
for any two time steps $s,s'$ in the current window $W = W(t)$, from which it easily follows that 
\begin{equation}
 |\mu_t - \mu_W| \le 3 b N + \tilO\left(\sqrt{\frac{1}{N}}\right).
\end{equation}
\end{lem}

Lemma \ref{lem:gradualbound} is a strong characterization because $\mu_s - \mu_{s'}$ does not depend on the window size $|W|$: No matter how long the current window is, $\mu_s - \mu_{s'}$ and thus $\mu_t - \mu_W$ is bounded in terms of change speed $b$: In other words, if $\mu_t - \mu_W$ is larger, ADWIN shrinks the window.
Note that the argument here is more general than the (informal) discussion in \cite{bifet2007learning} for gradual drift. The discussion in \cite{bifet2007learning} is limited to the case of linear drift, whereas our result holds for arbitrary drift as long as it moves at most $b$ per time step. 

\begin{proofof}{Proof of Lemma \ref{lem:gradualbound}}
We first consider the case $|W| = C N$ for some integer $C \in \Natural^+$. We decompose $W$ into $C$ subwindows of equal size $N$ and let $W_c$ be the $c$-th subwindow for $c \in [C]$. 
For $c \in [C] \setminus \{1\}$, let $W_{:c}$ be the joint subwindow of $W$ before $W_c$. Namely, $W_{:c} = W_1 \cup W_2 \cup \dots \cup W_{c-1}$. The fact that the window grows to size $W$ without a detection implies that each split $W_{:c} \cup W_c$ satisfies
\begin{equation}\label{ineq:wadiff}
|\hatmu_{ W_{:c} } - \hatmu_{ W_{c} }| 
\le 2 \sqrt{\frac{\log (T^3)}{N}}.
\end{equation}
Let $c \in [C]$ be arbitrary.
By recursively applying Eq.\eqref{ineq:wadiff}
we have
\begin{align}
|\hatmu_W - \hatmu_{W_{c}}|
&= 
\left|
\frac{1}{C} \hatmu_{ W_C } + \frac{C-1}{C} \hatmu_{ W_{:C} } - \hatmu_{W_{c}} 
\right|
\\
&\le \left|
\hatmu_{ W_{:C} } - \hatmu_{W_{c}} 
\right| 
+ \frac{2}{C} \sqrt{\frac{\log (T^3)}{N}} 
\text{\ \ \ \ \ (by Eq.\eqref{ineq:wadiff})}
\\
&\le \left|
\hatmu_{ W_{:C-1} } - \hatmu_{W_{c}} 
\right| 
+ \left(\frac{2}{C-1} + \frac{2}{C}\right) \sqrt{\frac{\log ( T^3)}{N}} 
\\ 
&\dots\\
&\le \left|
\hatmu_{ W_{:c} } - \hatmu_{W_{c}} 
\right| 
+ \sum_{c'=c+1}^C \frac{2}{c'} \sqrt{\frac{\log (T^3)}{N}} \\
&\le \left(1 + \sum_{c'=c+1}^C \frac{2}{c'} \right) \sqrt{\frac{\log (T^3)}{N}} 
\text{\ \ \ \ \ (by Eq.\eqref{ineq:wadiff})} \\
&\le (\log T) \sqrt{\frac{\log (T^3)}{N}},
\end{align}
which implies for any $c, c' \in [C]$ we have
\begin{equation}\label{ineq:wadiff_all}
|\hatmu_{W_{c}} - \hatmu_{W_{c'}}|
\le 2 (\log T) \sqrt{\frac{\log (T^3)}{N}}.
\end{equation}

By Eq.~\eqref{ineq:allwindowsbound} we have
\begin{align}\label{ineq:conf_eachregion_rev}
|\mu_{ W_{c} } - \hatmu_{ W_{c} }| &\le \sqrt{\frac{\log (T^3)}{N}} \nn
|\mu_{ W_{c'} } - \hatmu_{ W_{c'} }| &\le \sqrt{\frac{\log (T^3)}{N}}.
\end{align}
By the fact that $\mu_t$ moves $bN$ within a subwindow of size $w$, for any $s \in W_{c}, s' \in W_{c'}$, we have
\begin{align}\label{ineq:drift_each_rev}
|\mu_{ W_{c} } - \mu_{s}| &\le bN \nn
|\mu_{ W_{c'} } - \mu_{s'}| &\le bN.
\end{align}
By using these, we have
\begin{align}
|\mu_{s} - \mu_{s'}|
&\le
|\mu_{ W_{c} } - \mu_{s}| + |\mu_{ W_{c'} } - \mu_{s'}|
+ |\mu_{ W_{c} } - \hatmu_{ W_{c} }| + |\mu_{ W_{c'} } - \hatmu_{ W_{c'} }| +
|\hatmu_{W_{c}} - \hatmu_{W_{c'}}| \\
&\le 2 bN + 2 (1 + \log T) \sqrt{\frac{\log (T^3)}{N}}.
\text{\ \ \ \ \ (by Eq.\eqref{ineq:wadiff_all},\eqref{ineq:conf_eachregion_rev},\eqref{ineq:drift_each_rev})}
\end{align}
The general case of $|W| = CN + n$ for $n \in \{0,1,\dots,N-1\}$ is easily proven by replacing $2bN$ with $3bN$ since $\mu_t$ drift at most $bN$ in $n$ time steps.
\end{proofof} %

\subsection{Lemma \ref{lem:errorshrink}}

Lemma \ref{lem:gradualbound} characterizes the accuracy of the estimator. However, Lemma \ref{lem:gradualbound} only holds when window size $|W| \ge N$.
When ADWIN shrinks the window very frequently, we cannot guarantee the quality of the estimator $\hatmu_W$. The following Lemma \ref{lem:errorshrink} states that this is not the case: With high probability, the current window grows until $|W| = O(b^{-2/3})$.
\begin{lem}{\rm (Bound on erroneous shrinking)}\label{lem:errorshrink}
Let $C_1 = \tilO(1)$ be a sufficiently large value that is later defined in Eq.~\eqref{ineq:ndef}. Let the parameter of ADWIN be $\delta = 1/T^3$. 
Let the drift speed $b$ be such that $C_1 b^{-2/3} \le T$.
Let  
\begin{multline}
\EP(t) = \bigcup_{W_1, W_2 : W(t) = W_1 \cup W_2} 
\left\{ |W_1| \le C_1 b^{-2/3}, |W_2| \le C_1 b^{-2/3},  |\hatmu_{W_1} - \hatmu_{W_2}| \ge  \epscut^\delta \right\}.
\end{multline}
Let 
\begin{equation}\label{ineq:grad_growing}
\EP = \bigcup_{t \in [T]} \EP(t).
\end{equation}
Then,
\begin{equation}
\Pr[\EP] \le 2 C_1 T^{-1}.   
\end{equation}
\end{lem}

\begin{proofof}{Proof of Lemma \ref{lem:errorshrink}}
Let $C_1 = \tilO(1)$ that we define later in Eq.~\eqref{ineq:ndef}.
Let 
\begin{equation}
\EW_{C_1} = \{W' \in \EW: |W'| \le C_1 b^{-2/3}\}|
\end{equation}
be the subset of the window set $\EW$ with their sizes at most $C_1 b^{-2/3}$. It is easy to show that $|\EW_{C_1}| \le T C_1 b^{-2/3}$. Let $d$ be such that $T^{-d} = b$.
Similarly to Eq.~\eqref{ineq:allwindowsbound}, by the union bound of the Hoeffding bound over all windows of $\EW_{C_1}$, with probability at least
\begin{equation}
    1 - \frac{2}{T^{2+d}} \times T C_1 b^{-2/3} = 1 - 2 C_1 T^{-1} b^{1/3} \ge 1 - 2 C_1 T^{-1},
\end{equation}
we have
\begin{align} \label{ineq:gradual_oneplusc}
|\mu_{W_1} - \hatmu_{W_1}| &\le \sqrt{\frac{\log ( T^{2+d})}{2|W_1|}} \nn
|\mu_{W_2} - \hatmu_{W_2}| &\le \sqrt{\frac{\log ( T^{2+d})}{2|W_2|}}
\end{align}
holds for all $t$ and any split $W_1 \cup W_2 = W(t): |W_1|, |W_2| \le C_1 b^{-2/3}$. Let $N = C_1 b^{-2/3}$.
By definition of gradually changing stream, 
\begin{equation}\label{ineq:gradual_driftwidth}
|\mu_{W_1} - \mu_{W_2}|
\le
2bN.
\end{equation}

\updated{
In the following, we show that $|W_1| \le C_1 b^{-2/3}, |W_2| \le C_1 b^{-2/3},  |\hatmu_{W_1} - \hatmu_{W_2}| \ge  \epscut^\delta$ never occur under Eq.~\eqref{ineq:gradual_oneplusc} by using a proof-by-contradiction argument.
Namely,}
\begin{align}
\sqrt{ \frac{\log (T^3)}{2|W_1|} } + \sqrt{ \frac{\log (T^3)}{2|W_2|} } 
& \le |\hatmu_{W_1} - \hatmu_{W_2}| 
\text{\ \ \ (when $|\hatmu_{W_1} - \hatmu_{W_2}| \ge  \epscut^\delta$)} \\ 
& \le |\mu_{W_1} - \mu_{W_2}| + |\hatmu_{W_1} - \mu_{W_1}| + |\hatmu_{W_2} - \mu_{W_2}|  \text{\ \ \ (by triangular inequality)} \\
& \le 2bN + |\hatmu_{W_1} - \mu_{W_1}| + |\hatmu_{W_2} - \mu_{W_2}| 
\text{\ \ \ (by \eqref{ineq:gradual_driftwidth})} \\
& \le 2bN + \sqrt{\frac{\log (T^{2+d})}{2|W_1|}} + \sqrt{\frac{\log (T^{2+d})}{2|W_2|}} 
\text{\ \ \ (by \eqref{ineq:gradual_oneplusc})} 
\end{align}
which implies
\begin{equation}
\frac{\sqrt{3}-\sqrt{2+d}}{\sqrt{2}}
\left( 
\sqrt{\frac{\log T}{|W_1|}} +
\sqrt{\frac{\log T}{|W_2|}} 
\right)
\le 
2bN,
\end{equation}
which does not hold for 
\begin{equation}\label{ineq:ndef}
N = b^{-2/3}\underbrace{\left( \frac{\sqrt{3}-\sqrt{2+d}}{2\sqrt{2}} \sqrt{\log T} \right)^{2/3}}_{=: C_1}, |W_1|, |W_2| \le N.
\end{equation}
In summary, with probability $1 - 2 C_1 T^{-1}$, we have $\EP^c$.
\end{proofof}

\subsection{Proof of Theorem \ref{thm:adwingradual}}
\label{subsec:proof_adwingradual}
Let $d$ be such that $b = T^{-d}$.
Let $C_1$ be the value defined in Eq.~\eqref{ineq:ndef} 
and $N_1 = C_1 b^{-2/3}$. 

By using Lemma \ref{lem:errorshrink}, we bound the number of shrinks.
\begin{lem}{\rm (Number of shrinks)}\label{lem:graduannumshrink}
Let 
\begin{equation}
\ET_d^{\mathrm{sml}} := \{t \in \ET_d, |W_2(t)| < N_1\}.
\end{equation}
Then, under $\EP^c$, $|\ET_d^{\mathrm{sml}}| \le T/N_1$ holds.
\end{lem}
\begin{proofof}{Proof of Lemma \ref{lem:graduannumshrink}}
$\EP^c$ implies that for any shrink $|W_1| \ge N_1$ or $|W_2| \ge N_1$ holds. A shrink of the latter case is not included in $\ET_d^{\mathrm{sml}}$. A shrink of the former case reduces the size of window at least $N_1$, and cannot occur more than $T/N_1$ times.
\end{proofof}

\begin{proofof}{Proof of Theorem \ref{thm:adwingradual}} %
Let $\cbreak(t) = t - \max_{s < t, s \in \ET_d} s$ be the number of time steps after the last detection time.
We have
\begin{align}\label{ineq:adwingradual_main} %
\Err(T) 
&\le T \Ind[\EP] + \sum_{n=1}^{N_1} \sum_{t=1}^T |\mu_t - \hatmu_W| \Ind[\cbreak(t)=n, \EP^c] + \sum_{t=1}^T |\mu_t - \hatmu_W| \Ind[\cbreak(t) \ge N_1, \EP^c].
\end{align}
By Lemma \ref{lem:errorshrink}, the expectation of the first term of Eq.~\eqref{ineq:adwingradual_main} is bounded as 
\begin{equation}
\EP[T \Ind[\EP]] = T \times 2 C_1 T^{-1} = \tilO(1).
\end{equation}
We next bound the second term of Eq.~\eqref{ineq:adwingradual_main}. By Eq.~\eqref{ineq:allwindowsbound}, 
\begin{equation}\label{ineq:adgradual_hoeffding}
\forall_{t \in [T]} 
\left\{ 
|\mu_t - \hatmu_W| \le \sqrt{ \frac{\log (T^3)}{2(\cbreak(t) - 1)} } + b \cbreak(t)
\right\}
\end{equation}
holds with probability $1-2/T$.
The expectation of the second term is bounded as:
\begin{align}
\lefteqn{
\sum_{n=1}^{N_1} \sum_{t=1}^T \Pr[ |\mu_t - \hatmu_W| \Ind[c(t)=n, \EX^c] ]
}\\
&\le
(|\ET_d^{\mathrm{sml}}|+1) \sum_{n=1}^{N_1} 
\left(
\sqrt{ \frac{\log (T^3)}{2n} } + b n
\right)
+ T \times \frac{2}{T}
\intertext{\ \ \ (by Eq.~\eqref{ineq:adgradual_hoeffding} and the fact that $\cbreak(t)=n$ for each $n$ occurs at most once between two detection times)} 
&\le \left(\frac{T}{N_1} + 1 \right) 
\sum_{n=1}^{N_1} 
\left(
\sqrt{ \frac{\log (T^3)}{2n} } + b n
\right) + 2
\\
&\text{\ \ \ (by Lemma \ref{lem:graduannumshrink})} \\
&\le \tilO\left(\frac{T}{N_1} \times (\sqrt{N_1} + b N_1^2)\right) = O(T^{1-d/3}) 
\end{align}
where we used $b, N_1 = T^{-d}, \tilO(T^{(2d)/3})$ in the last transformation. 
Moreover, by using Lemma \ref{lem:gradualbound}, the expectation of the third term is bounded as
\begin{align}
\sum_{t=1}^T \Ex[ |\mu_t - \hatmu_W| \Ind[c(t) \ge N_1] ]
\le T \times \left(3 b N_1 +  \tilO\left(\sqrt{\frac{1}{N_1}}\right) \right) 
= \tilO(T^{1-d/3}),
\end{align}
which completes the proof.
\end{proofof} 

\section{Proofs of MAB}

We first clarify the notation: In bandit streams, let
\begin{equation}
W^i = \{s \in W: i = I(s)\}.
\end{equation}
Let
\begin{equation}
\mu_{i,W} = \sum_{s \in W^i} \mu_{i,s} 
\end{equation}
and its empirical estimate be 
\begin{equation}
\hatmu_{i,W} = \sum_{s \in W^i} x_{i,s}.
\end{equation}
In the bandit setting, only one arm $I(t) \in [K]$ is observable at round $t$, and thus $W^i \subset W$. Still, the following inequality holds with probability $1 - 2/T^p$:
\begin{equation}\label{ineq:allwindowsbound_bandit}
\forall\ {W' \in \EW}\ |\mu_{i,W'} - \hatmu_{i,W'}| \le \sqrt{\frac{\log (T^{2+p})}{2 |(W')^i|}}.
\end{equation}
Eq.~\eqref{ineq:allwindowsbound_bandit}, which is the same form as Eq.~\eqref{ineq:allwindowsbound}, holds because it corrects $T^2$ multiplications. The union bound of Eq.~\eqref{ineq:allwindowsbound_bandit} over all arms holds with $1 - K/T^p$.

\subsection{Proof of Theorem \ref{thm:regret_stationary}}
\label{subsec:regstationary}

\begin{proofof}{Proof of Theorem \ref{thm:regret_stationary}} %
Under stationary environment, Eq.~\eqref{ineq:allwindowsbound_bandit} with $p=1$ implies that \AADWIN{}-bandit never reset itself %
with high probability. 
The regret is bounded by
\begin{align}
\Reg(T) 
&\le \underbrace{T \times \frac{2K}{T}}_{\text{regret of false reset}} + \underbrace{\sum_i \Cst \frac{\log{T}}{\Delta_i}}_{\text{regret of \AADWIN{}-bandit, Lemma \ref{lem_monitoringdtprop}}} \\ 
&\le \Cst \sum_i \frac{\log{T}}{\Delta_i} + O(1). %
\end{align}
\end{proofof}

\subsection{Analysis of Thompson sampling under drift
}
\label{subsec:proof_ts}

This section derives a proof of Lemma \ref{lem_tsdt}.
Without loss of generality, we define $\mu_{1,1} \ge \mu_{2,1} \ge \dots \ge \mu_{K,1}$ at round $1$. 
For ease of notation, let $\mu_i := \mu_{i,1}$ and $\eps = \eps(T)$.

In the following, we derive the drift-tolerant regret of TS. 

\begin{lem}{\rm (Number of suboptimal draws)}\label{lem_ts_suboptdraw}
Let $c > 0$ be arbitrary. 
Assume that we run TS.
The following inequality holds
\begin{equation}
\sum_{t \in [T]} %
\Pr[
\tilmu_{I(t),t} \le \mu_{1,1} - \eps(t) - c
]
= O\left(\frac{1}{c^2}\right).
\end{equation}
\end{lem}
\begin{proofof}{Proof of Lemma \ref{lem_ts_suboptdraw}}

We have, 
\begin{align}
\sum_{t} \Ind[ \tilmu_{I(t),t} \le \mu_{1,1} - \eps(t) - c
] 
&= \sum_{n=0}^T \sum_{t=1}^T \Ind[ 
\tilmu_{I(t),t} \le \mu_{1,1} - \eps(t) - c,
N_1(t) = n
]\\
&= \sum_{n=0}^T \sum_{m=1}^T
\Ind\left[
\sum_{t} \Ind\left[
\tilmu_{I(t),t} \le \mu_{1,1} - \eps(t) - c,
N_1(t) = n
\right]
\ge m
\right].
\end{align}
Here, if
\begin{equation}
\left\{
\tilmu_1(t) > \mu_{1,1} - \eps(t) - c,\,
\max_{i \ne 1} \tilmu_{i,t} \le \mu_{1,1} - \eps(t) - c,\,
N_1(t) = n
\right\}
\end{equation}
then arm $1$ is drawn. This fact implies that event $
\{
\sum_{t=1}^T \Ind\left[
\tilmu_{I(t),t} \le \mu_{1,1} - \eps(t) - c,
N_1(t) = n
\right]
\ge m
\}$
requires that 
\[
\tilmu_1(t) \le \mu_{1,1} - \eps(t) - c
\]
holds in the first $m$ rounds of 
the subsequence $\tau_1, \tau_2, \tau_3, \dots, \tau_m = \{t: \max_{i \ne 1} \tilmu_{i,t} \le \mu_{1,1} - \eps(t) - c,\,
N_1(t) = n\}$. 
Letting $S_{\beta}(x; \mu, n)$ be the survival function\footnote{The probability $\theta \sim \Beta(1+\mu n, 1+(1-\mu)n)$ exceeds $x$.} of $\Beta(1+\mu n, 1+(1-\mu)n)$, we have
\begin{align}
\lefteqn{
\sum_{n=0}^T
\Ex\left[ \sum_{m=1}^T
\sum_{t=1}^T \Ind\left[
\tilmu_{I(t),t} \le \mu_{1,1} - \eps(t) - c,
N_1(t) = n
\right]
\ge m
\right]
}\\
&\le 
\sum_{n=0}^T 
\Ex\left[
\sum_{m=1}^T
\prod_{k=1}^m
\Pr\left[
\tilmu_{1,\tau_k} \le \mu_{1,1} - \eps(\tau_k) - c |
\hatmunth_1
\right]
\right]\\
&=
\sum_{n=0}^T 
\Ex\left[
\sum_{m=1}^T
\left(
S_\beta(\mu_{1,1} - \eps(\tau_1) - c; \hatmunth_1, n)
\right)^m
\right]\text{\ \ \ \ ($\eps(\tau_k)$ is nondecreasing in $k$)}\\
&\le
\sum_{n=0}^T 
\Ex\left[
\frac{1 - S_\beta(\mu_{1,1} - \eps(\tau_1) - c; \hatmunth_1, n)}{S_\beta(\mu_{1,1} - \eps(\tau_1) - c; \hatmunth_1, n)}
\right]\\
&= O(1/c^2).
\text{\ \ \ \ (by Lemma \ref{lem_ts_underestimation})} 
\end{align}
where, to apply Lemma \ref{lem_ts_underestimation}, we used the fact that $\hatmunth_1$ is the mean of $n$ independent samples from Bernoulli distributions with means $\mu_{1,1} - \eps(\tau_1)$ or more.
\end{proofof} %

\subsubsection{Proof of Lemma \ref{lem_tsdt}}\label{proof_lem_tsdt}

Let $(x)_+ = \max(x, 0)$.
We have,
\begin{align}
\lefteqn{
\Regbasetor(T, 3)
}\\
&= \sum_{t} (\reg(t) - 3\eps(t))_+\\
&=
\sum_{t} \sum_{i\in[K]} 
\left(\Delta_i - 3\eps(t)\right)_+
\Ind\left[
I(t)=i
\right]\\
&\le
\sum_{t} \sum_{i\in[K]} 
\left(\Delta_i - 3\eps(t)\right)_+
\left(
\Ind\left[
I(t)=i,
\tilmu_{I(t),t} < \frac{\mu_1+\mu_i}{2}
\right]
+
\Ind\left[
I(t)=i,
\tilmu_{I(t),t} \ge \frac{\mu_1+\mu_i}{2}
\right]
\right)
\\
&=
\sum_{t} \sum_{i\in[K]} 
\left(\Delta_i - 3\eps(t)\right)_+
\left(
\Ind\left[
I(t)=i,
\tilmu_{I(t),t} < \mu_1 - \frac{\Delta_i}{2}
\right]
+
\Ind\left[
I(t)=i,
\tilmu_{I(t),t} \ge \mu_i + \frac{\Delta_i}{2}
\right]
\right)
\\
&\le
\sum_{t} \sum_{i\in[K]} 
\left(\Delta_i - 3\eps(t)\right)_+
\left(
\Ind\left[
I(t)=i,
\tilmu_{I(t),t} < \mu_1 - \eps(t) - \frac{\Delta_i}{6}
\right]
+
\Ind\left[
I(t)=i,
\tilmu_{I(t),t} \ge \mu_i + \eps(t) + \frac{\Delta_i}{6}
\right]
\right).
\nn
&\phantom{wwwwwwwwwwwwwwwwwwwwww}\text{\ \ \ \ (by $\Delta_i \ge 3 \eps(t)$ or $(\Delta_i - 3\eps(t))_+ = 0$)}
\end{align}
The first term is bounded in expectation as
\[
\sum_{t} \sum_{i \in [K]} \Delta_i \Ex\left[
I(t)=i, \tilmu_{I(t),t} \le \mu_1 - \eps(t) - \frac{\Delta_i}{6}\right]
= \sum_{i \in [K]} \Delta_i \times O\left( \frac{1}{\Delta_i^2} \right) 
= \sum_{i \in [K]} O\left( \frac{1}{\Delta_i} \right),
\]
where we used Lemma \ref{lem_ts_suboptdraw} in the first transformation with $c = \Delta_i/3$.

We next bound the second term.
\begin{align}
\lefteqn{
\sum_{t} \sum_{i \in [K]} \Delta_i 
\Ind\left[I(t)=i, 
\tilmu_{I(t),t} > \mu_1 + \eps(t) + \frac{\Delta_i}{6}
\right] 
}\\
&\le \sum_{i \in [K]} \left(
\frac{144 \log T}{\Delta_i}
+
\sum_{t} \Delta_i 
\Ind\left[\tilmu_{i,t} \ge \mu_i + \eps(t) + \frac{\Delta_i}{6}, \Nit \ge \frac{144 \log T}{\Delta_i^2}
\right]
\right).
\end{align}
Here,
\begin{align}
\lefteqn{
\sum_{t} \Delta_i 
\Ind\left[\tilmu_{i,t} \ge \mu_i + \eps(t) + \frac{\Delta_i}{6}, \Nit \ge \frac{144 \log T}{\Delta_i^2}
\right]
}\\
&\le 
\sum_{t} \Delta_i 
\Ind\left[ \hatmu_{i,t} \ge \mu_i + \eps(t) + \frac{\Delta_i}{12}, \Nit \ge \frac{144 \log T}{\Delta_i^2}
\right]
\nn &\ \ \ \ \ \ + 
\sum_{t} \Delta_i 
\Ind\left[\tilmu_{i,t} \ge \mu_i + \eps(t) + \frac{\Delta_i}{6}, \hatmu_{i,t} \le \mu_i + \eps(t) + \frac{\Delta_i}{12},  \Nit \ge \frac{144\log T}{\Delta_i^2}
\right]. \label{ineq_ts_overest_two}
\end{align}
By using the Hoeffding inequality, the first term of Eq.~\eqref{ineq_ts_overest_two} is bounded as
\begin{align}
\lefteqn{
\sum_{t} \Delta_i 
\Ind\left[ \hatmu_{i,t} \ge \mu_i + \eps(t) + \frac{\Delta_i}{12}, \Nit \ge \frac{144 \log T}{\Delta_i^2}
\right]}\\
&\le \sum_{N} T \exp\left(- 2 \times \frac{144 \log T}{\Delta_i^2} \times \left(\frac{\Delta_i}{12}\right)^2\right) \\
&\le T^2 \exp\left(- 2 \log T\right) = 1,
\end{align}
where we have used the Hoeffding inequality and the fact that $\hatmu_{i,t}$ is the mean of $\Nit$ samples with each mean no more than $\mu_i + \eps(t)$.
By using Lemma \ref{lem_posterior}, the second term of Eq.~\eqref{ineq_ts_overest_two} is bounded as 
\begin{multline}
\sum_{t} \Delta_i 
\Ex\left[\tilmu_{i,t} \ge \mu_i + \eps(t) + \frac{\Delta_i}{6}, \hatmu_{i,t} \le \mu_i + \eps(t) + \frac{\Delta_i}{12},  \Nit \ge \frac{144 \log T}{\Delta_i^2}
\right] 
\\
\le \sum_{n,t} \exp\left(
- 2 
\times \frac{144 \log T}{\Delta_i^2} 
\times \left(\frac{\Delta_i}{12}\right)^2
\right)
\le 1.
\end{multline}

\subsubsection{Lemmas for Thompson sampling}

The following lemma is a version of Lemma 4 in \cite{agrawalaistats13}.
\begin{lem}\label{lem_bound_recovery}
Let $\hatmunth_1$ is the mean of $n$ independent samples that are drawn from Bernoulli distributions with their means no less than $\mu_1 - \eps(t)$. Then,
\begin{equation}
\Ex\left[
\frac{1}{S_\beta(\mu_1 - \eps(t) - c; \hatmunth_1, n)}
\right]
\le 
\left\{
\begin{array}{ll}
1 + \frac{3}{c} & n < \frac{8}{c} \\
1 + \Theta\left(
e^{-c^2 n /2}
+ \frac{1}{(n+1)c^2} e^{-D_i n}
+ \frac{1}{e^{c^2 n/ 4} - 1}
\right) & n \ge \frac{8}{c}
\end{array}
\right.
,
\end{equation}
where $D_i = 2 c^2$.
\end{lem}

\begin{lem}\label{lem_ts_underestimation}
Let $\hatmunth_1$ is the mean of $n$ independent samples that are drawn from Bernoulli distributions with their means no less than $\mu_1 - \eps(t)$.
Then, the following inequality holds:
\begin{align}
\lefteqn{
\sum_{n=0}^T 
\Ex\left[
\frac{1 - S_\beta(\mu_1 - \eps(t) - c; \hatmunth_1, n)}{S_\beta(\mu_1 - \eps(t) - c; \hatmunth_1, n)}
\right]
}\\
&\le 
\frac{24}{c^2} + \sum_{n=0}^T \Theta\left(
e^{-c^2 n /2}
+ \frac{1}{(n+1)c^2} e^{-D_i n}
+ \frac{1}{e^{c^2 n/ 4} - 1}
\right)
\\
&= O\left(\frac{1}{c^2}\right).
\end{align}
\end{lem}
\begin{proofof}{Proof of Lemma \ref{lem_ts_underestimation}}
The proof of Lemma \ref{lem_ts_underestimation} is straightforward from Lemma \ref{lem_bound_recovery} by following the similar steps to the problem-independent bound of Theorem 2 in \cite{agrawalaistats13}.
\end{proofof}

The following lemma is well-known in the Beta posterior.
\begin{lem}\label{lem_posterior}
\begin{align}
\Pr[
\tilmu_{i,t} > \hatmu_{i,t} + c 
\,|\,
\hatmu_{i,t}, \Nit = N
]
&\le 
\exp\left(
- 2 N c^2
\right)
\end{align}
for any $c>0$.%
\end{lem}
\begin{proofof}{Proof of Lemma \ref{lem_posterior}}
\begin{align}
\lefteqn{
\Pr[
\tilmu_{i,t} > \hatmu_{i,t} + c 
\,|\,
\hatmu_{i,t}, \Nit = N
]
}\nn
&\le e^{-N \KL(\hatmu_{i,t}, \hatmu_{i,t}+c)} \text{\ \ \ \ (Lemma 3 in \cite{agrawalaistats13})}\nn
&\le e^{-N c^2}. \text{\ \ \ \ (by Pinsker's inequality)}
\end{align}
\end{proofof}

\subsection{Regret due to Monitoring}\label{subsec:proof_monitoring}

\begin{proofof}{Proof of Lemma \ref{lem_monitoringreg}}
We call each consecutive $\wasD N$ rounds as ``subblock''. 
For each block $l=1,2,\dots$, the number of subblocks is at most $2^{l-1}$. 
The $l$-th monitoring arm $i^{(l)}$ is drawn during $2^l + 1$ subblocks, and at each subblock it is drawn for $N$ rounds (Figure \ref{fig:monitoring}).

The goal of this theorem is bound the ratio between the regret during monitoring rounds divided by the regret during base-bandit rounds. 
Let 
\begin{equation}
\ETmonitor{l} = \{
t: \text{arm $i^{(l)}$ is drawn as the $l$-th monitoring arm}
\}
\},
\end{equation}
as illustrated in Figure \ref{fig:monitoring}. By definition,
\begin{equation}\label{ineq_monitoringsum}
\sum_t \reg(t) 
= \sum_{t \in \ETbase} \reg(t) 
+ \sum_l \sum_{t \in \ETmonitor{l}} \reg(t)
= \sum_l 
\left(\sum_{t \in \ETbase \cap \ET_l} \reg(t) 
+ \sum_{t \in \ETmonitor{l}} \reg(t)
\right),
\end{equation}
where $\ET_l$ is the set of rounds in block $l$.
First, for each monitoring arm $i^{(l)}$, we bound the ratio between the number of draws of the monitoring arms divided by the number of draws of base-bandit arms. 
Namely, 

\begin{equation}\label{ineq_ratio_l}
\frac{
\sum_{t} \Ind[t \in \ETmonitor{l}]
}{
\sum_{t} \Ind[t \in \ETbase \cap \ET_l] \Ind[I(t) = i^{(l)}]
},
\end{equation}
for each $l$.
First, we consider the ratio of the regret due to the first monitoring arm $i^{(1)}$.
The arm $i^{(1)}$ is the most drawn arm in the first subblock, which is drawn at least $\wasD N/K = N$ rounds at $l=1$ by the base-bandit algorithm. We draw $i^{(1)}$ for $2 N$ times during block $l=2$ as a monitoring arm. Therefore, Eq.~\eqref{ineq_ratio_l} for $l=1$ is at most two.

In the following, we consider the Eq.~\eqref{ineq_ratio_l} for $l \ge 2$.
Since the arm most frequently pulled in the first $(2^l - 2) N$ subblocks is chosen as $i^{(l)}$, it has been drawn at least $(2^l - 2) N$ times by the base-bandit algorithm before the block $l$. 
We draw this arm for $(2^l + 1) N$ times as a monitoring arm.
Therefore, the ratio is at most
\begin{align}
\frac{(2^l + 1)N}{ (2^l - 2) N }
\le
3.
\label{ineq_monitor_multiple}
\end{align}
In summary, Eq.~\eqref{ineq_ratio_l} is at most $3$. In the following, we bound the total regret due to the regret during rounds $\ETbase$.
Namely,
\begin{align}
\sum_{t \le S} \reg(t)
&=
\sum_l 
\left(\sum_{t \in \ETbase \cap \ET_l} \reg(t) 
+ \sum_{t \in \ETmonitor{l}} \reg(t)
\right) \text{\ \ \ \ (by Eq.~\eqref{ineq_monitoringsum})}\\
&=
\sum_l 
\left(\sum_{t \in \ETbase \cap \ET_l} \Delta_{I(t)} 
+ \sum_{t \in \ETmonitor{l}} \Delta_{i^{(l)}} 
\right) + \sum_t \eps(t)\\
&\le
\sum_l 
\left(\sum_{t \in \ETbase \cap \ET_l} \Delta_{I(t)} 
+ 3 \sum_{t \in \ETbase \cap \ET_l} \Ind[I(t)=i^{(l)}]\Delta_{i^{(l)}} 
\right) + \sum_t \eps(t)
\text{\ \ \ \ (by Eq.~\eqref{ineq_ratio_l})}\\
&\le
4 \sum_l 
\sum_{t \in \ETbase \cap \ET_l} \Delta_{I(t)} 
 + \sum_t \eps(t)\\
&=
4 \sum_{t \in \ETbase} \Delta_{I(t)} 
 + \sum_t \eps(t)\\
&\le
4 \sum_{t \in \ETbase} \reg(t) 
 + 5 \sum_t \eps(t).
\end{align}

\komiyama{(The following is obsolete, will delete)}
Let $\Nbase_l$ be the number of draws of arm $i^{(l)}$ at block $l$.
\updated{By the fact that the regret due to draws of arm $i^{(l)}$ at round $t$ is at most $\Delta_{i^{(l)}} + \eps(t)$ and one arm is drawn for each round}, the regret due to monitoring rounds is at most 
\begin{align}
3\sum_{l\ge 1} \Nbase_l \Delta_{i^{(l)}} + \sum_t \eps(t)
\le 3 \sum_{t \in \ETbase} \reg(t) + \sum_t 2 \eps(t), 
\end{align}
\updated{
by using the fact that the regret during the base-bandit rounds is at least $\Nbase_l \Delta_{i^{(l)}} - \sum_t \eps(t)$.
The total regret, which is the sum of that of base-bandit rounds and monitoring rounds, is bounded as
\begin{equation}
\sum_{t=1}^S \reg(t) 
\le 
4
\left(
\sum_{t \in \ETbase} \reg(t) 
\right)
+ \sum_{t=1}^S 3\eps(t)
\end{equation}
which concludes the proof.
}
\end{proofof}

\subsection{Regret in abrupt environment: proof of Theorem \ref{thm:regret_abrupt}}
\label{subsec:regret_abrupt}

\begin{proofof}{Proof of Theorem \ref{thm:regret_abrupt}}
Similar to the case of single stream, we define detection times, which is the subset of rounds where the reset occurs.
Let $\ET_d = \{t \in [T]: |W(t+1)| =0\}$ be the set of detection times, and $M_d = |\ET_d|$ be the number of detection times\footnote{Remember the difference between changepoints $\ET$ and detection times $\ET_d$. The former is defined on an abrupt environment, whereas the latter is defined for the \AADWIN{}-bandit algorithm (and thus the latter is a random variable).}.
Let $\Td{m}$ be the $m$-th element of $\ET_d$. For convenience, let $(\Td{0},\Td{M_d+1}) = (0, T)$.
We denote $\ES_m = \{\Td{m}+1,\Td{m}+2,\dots,\Td{m+1}\}$ for $m \in \{0,\dots,M_d\}$, which 
is the interval between $m$-th and ($m+1$)-th detection times.
We write
$S_m = |\ES_m|$. 

We define an event 
\begin{equation}
\EV = \{\forall m \in [\NumChange]\ \exists s \in \ET_d\ 0 \le s - T_m \le 16\wasD U(\epschg{m})\} \cap \{M_d = M\}.
\end{equation}
Event $\EV$ states that for each changepoint $T_m \in \ET_c$ (indexed by $m=1,2,\dots,M=|\ET_c|$), there exists a corresponding detection time $\Td{m}$ within $16\wasD U(\epschg{m})$ time steps.
We first show that $\EV$ holds with high probability by induction. 
Assume that for every changepoint $m'$ up to $1,2,\dots,m-1$, there exists unique detection time $s_{m'}$ such that $0 \le s_{m'} - T_{m'} \le 16\wasD U(\epschg{m})$. 
We show that $0 \le s_m - T_m \le 16\wasD U(\epschg{m})$. 
During $t \le T_m$ (these time steps are between $m$-th and $(m-1)$-th changepoint), $\mu_{i,t}$ stays the same. 
We denote $\mu_i = \mu_{i,t}$ during these time steps.
Eq.~\eqref{ineq:allwindowsbound_bandit} implies that, with probability $1 - 2K/T$, for any split $W(t) = W_1 \cup W_2$, 
\begin{align}
\forall_i\ 
|\mu_i - \hatmu_{i,W_1}| &\le \sqrt{\frac{\log (T^3)}{2 |W_{i,1}|}}\\
\forall_i\ 
|\mu_i - \hatmu_{i,W_2}| &\le \sqrt{\frac{\log (T^3)}{2 |W_{i,2}|}}
\end{align}
where $W_{i,1} = \{t \in W_1: i \in I(t)\}$ and $W_{i,2} = \{t \in W_2: i \in I(t)\}$. This implies \AADWIN{}-bandit with $\delta = T^3$ never makes a split before the $m$-th changepoint (i.e., $0 \le s_m - T_m$).
Let $s = T_m + 16 \wasD  U(\epschg{m})$. Assume that there is no detection between time step $T_{m-1} + 16 \wasD  U(\epschg{m})$ and $T_m + 16 \wasD  U(\epschg{m})$. Then for a split $W(s) = W_1 \cup W_2$, $W_1 = W \cap [T_m]$, $W_2 = W \setminus W_1$, we have
\begin{equation}\label{ineq:w1w2large}
|W_1| \ge (T_m - T_{m-1} - 16 \wasD  U(\epschg{m})),\qquad
|W_2| \ge 16 \wasD  U(\epschg{m}),
\end{equation}
\myred{
where we have used $s_{m-1} - T_{m-1} \le 16 \wasD  U(\epschg{m})$ in the first inequality.}
By assumption of Theorem \ref{thm:regret_abrupt}, the first inequality implies
\begin{align}
|W_1| \ge (T_m - T_{m-1} - 16 \wasD  U(\epschg{m})) \ge 3 \wasD N - 16 \wasD  U(\epschg{m}) \ge
2 \wasD N.
\end{align}
By the first property the monitoring consistency (Definition \ref{def:monitoring}),
there exists an arm $i \in [K]$ (i.e., monitoring arm $i^{(l)}$) such that 
\begin{equation}\label{ineq:wilargeu}
|W_{i,1}|, |W_{i,2}| \ge 16 U(\epschg{m}).
\end{equation}
Again by Eq.~\eqref{ineq:allwindowsbound_bandit} we have
\begin{align}\label{ineq:gapw1}
\forall_i\ 
|\mu_{i,T_m} - \hatmu_{i,W_1}| &\le \sqrt{\frac{\log (T^3)}{2 |W_{i,1}|}}\\ \label{ineq:gapw2}
\forall_i\ 
|\mu_{i,T_m+1} - \hatmu_{i,W_2}| &\le \sqrt{\frac{\log (T^3)}{2 |W_{i,2}|}}.
\end{align}
By definition of the global changepoint, 
\begin{equation}
|\mu_{i,T_m} - \mu_{i,T_m+1}| \ge \epschg{m} = \sqrt{\frac{\log (T^3)}{2 U(\epschg{m})}}.
\end{equation}

Combining these yields
\begin{align}
|\hatmu_{i,W_1} - \hatmu_{i,W_2}| 
&\ge |\mu_{i,T_m} - \mu_{i,T_m+1}|
- |\mu_{i,T_m} - \hatmu_{i,W_1}| - |\mu_{i,T_m+1} - \hatmu_{i,W_2}|\nn
&\phantom{wwwwwwwwwwwwwwwwwwww}\text{\ \ \ \ \ \ (by triangular inequality)}\\
&\ge \sqrt{\frac{\log (T^3)}{2 U(\epschg{m})}}
- |\mu_{i,T_m} - \hatmu_{i,W_1}| - |\mu_{i,T_m+1} - \hatmu_{i,W_2}| \\
&\ge 4 \max\left( \sqrt{\frac{\log (T^3)}{2 |W_{i,1}|}}, \sqrt{\frac{\log (T^3)}{2 |W_{i,2}|}}\right)
- |\mu_{i,T_m} - \hatmu_{i,W_1}| - |\mu_{i,T_m+1} - \hatmu_{i,W_2}|\nn
&\phantom{iwwwwwwwwwwwwwwwwwwww}\text{\ \ \ \ \ \ (by Eq.~\eqref{ineq:wilargeu})}\\
&\ge \sqrt{\frac{\log (T^3)}{2 |W_{i,1}|}} + \sqrt{\frac{\log (T^3)}{2 |W_{i,2}|}}. \phantom{awwww}\text{\ \ \ \ \ \ (by Eq.~\eqref{ineq:gapw1}, \eqref{ineq:gapw2})}
\end{align}
which implies that \AADWIN{}-bandit resets the window at round $s$.
In summary, 
under Eq.~\eqref{ineq:allwindowsbound_bandit} with $p=1$, $\EV$ holds. Therefore,
\begin{equation}
\Pr[\EV] \ge 1 - 2K/T.
\end{equation}

In the following, we bound the regret.
Let 
\begin{equation}
\Reg_m = \sum_{t=\Td{m}}^{\Td{m+1}} \reg(t).
\end{equation}
Namely, $\Reg_m$ corresponding to the regret between $m$-th and $(m+1)$-th detection times. 
The regret is decomposed as
\begin{align}
\Reg(T) 
&\le 
T \Ind[\EV^c] 
+ \Ind[\EV]\Reg(T)  \nn 
&\le 
T \Ind[\EV^c] + \Ind[\EV]\Reg(T)  \nn 
&=
T \Ind[\EV^c] + \Ind[\EV]\sum_{m=0}^M \Reg_m. 
\text{\ \ \ ($\EV$ implies $|\ET_d| = M$)}
\end{align}
The regret in the case of $\EV^c$ is bounded as
\begin{equation}
T \Pr[\EV^c] \le 2K.
\end{equation}

Under $\EV$, 
\begin{equation}\label{ineq:tbtcgap}
\Td{m+1} - \Td{m} \le T_{m+1} - T_{m} + 16 \wasD  U(\epschg{m}).
\end{equation}
Let $\mu_{i,m}$ be the mean of $i$-th arm between $m$-th changepoint and $(m+1)$-th changepoint.
Let $\Delta_{i,m} = \max_j \mu_{j,m} - \mu_i \ge 0$ be the corresponding gap.

By the definition of abrupt changes, $\eps(t) \le \Crabrupt \epschg{m}$ after the $m$-th changepoint and $0$ before the changepoint. 
By Lemma \ref{lem_monitoringdtprop} and Eq.~\eqref{ineq:tbtcgap}, we have
\begin{align}
\Ex\left[
\Ind[\EV] \sum_{m=0}^M \Reg_m 
\right]
&\le \sum_{m=0}^M 
\sum_{i: \Delta_{i,m} > 0} 
\left( 
\Cdr \frac{\log T}{\Delta_{i,m}}
+ 2 \Crabrupt \epschg{m} \times 16 \wasD  U(\epschg{m})
\right).\label{ineq_regabrupt_invdelta}
\end{align}
\updated{ %
Moreover, letting $\Nduring{i}{m} = \sum_{t=\Td{m}+1}^{\Td{m+1}} \Ind[I(t)=i]$, we also have
\[
\Ind[\EV] \Reg_m
\le 
\Delta_{i,m} \Nduring{i}{m} 
+
2 \Crabrupt \epschg{m} \times 16 \wasD  U(\epschg{m}),
\]
and thus
\begin{equation}\label{ineq_regabrupt_delta}
\Ex\left[
\Ind[\EV] \sum_{m=0}^M \Reg_m 
\right]
\le
\sum_{m=0}^M 
\sum_{i: \Delta_{i,m} > 0} 
\left( 
\Delta_{i,m} \Ex[
\Nduring{i}{m}
]
+ 2 \Crabrupt \epschg{m} \times 16 \wasD  U(\epschg{m})
\right)
\end{equation}
Taking the minimum of Eq.~\eqref{ineq_regabrupt_invdelta} and Eq.~\eqref{ineq_regabrupt_delta}, it holds that 
\begin{equation}
\Ex\left[
\Ind[\EV] \sum_{m=0}^M \Reg_m 
\right]
\le
\Ex\left[
\sum_{m=0}^M 
\sum_{i: \Delta_{i,m} > 0} 
\left( 
\max(1, \Cdr)
R_{i,m}
+ 2 \Crabrupt \epschg{m} \times 16 \wasD  U(\epschg{m})
\right)
\right],\label{ineq_abrupt_twoterms}
\end{equation}
where $R_{i,m} = \min(\Delta_{i,m} \Ex[\Nduring{i}{m}], \frac{\log T}{\Delta_{i,m}})$.
}

In the following, we bound the two terms of Eq.~\eqref{ineq_abrupt_twoterms} separately.

The first term of Eq.~\eqref{ineq_abrupt_twoterms} is bounded by using standard discussion of distribution-independent regret as follows. 
We have
\begin{align}
\sum_{m=0}^M \sum_i
R_{i,m}
&= \sum_{m=0}^M \sum_{i: \Delta_{i,m} > 0}
\sqrt{\Delta_{i,m} R_{i,m}} \sqrt{(\Delta_{i,m})^{-1} R_{i,m}}\\
&\le \sqrt{\log T} \sum_{m=0}^M \sum_{i: \Delta_{i,m} > 0} \sqrt{(\Delta_{i,m})^{-1} R_{i,m}}\\
&\le \sqrt{\log T} \sum_{m=0}^M \sum_{i: \Delta_{i,m} > 0} \sqrt{\Nduring{i}{m}}\\
&\le \sqrt{\log T} \sum_{m=0}^M \sqrt{K(T_{m+1}-T_m)}\nn
&\text{\ \ \ (by the Cauchy-Schwarz inequality and $\sum_i \Nduring{i}{m} = (\Td{m+1} - \Td{m})$)}\\
&\le \sqrt{\log T} \sqrt{KMT}.\nn
&
\text{\ \ \ (by the Cauchy-Schwarz inequality and $\sum_m (T_{m+1} - T_m) = T$)}
\end{align}

\komiyama{updated the following}
The second term of Eq.~\eqref{ineq_abrupt_twoterms} is bounded by a similar technique as follows.
\begin{align}
\lefteqn{
\sum_{m=0}^M 
\Crabrupt \eps_{m} 16 \wasD  U(\eps_{m})
}\\
&= \sum_{m=0}^M \Crabrupt \eps_{m} \min\left(16 \wasD  U(\eps_{m}), \frac{T_{m+1} - T_m}{3}\right)\nn
&\phantom{wwww}\text{\ \ \ \ \ (by $ (T_{m+1} - T_m) \ge 48 \wasD  U(\eps_{m})$ and $\min(a,b)=a$ if $a\le b$)}\\
&= \sum_{m=0}^M \Crabrupt \min\left(\frac{16\wasD \log (T^3)}{\eps_{m}}, \eps_{m} \frac{T_{m+1} - T_{m}}{3}\right) 
\nn
&\phantom{wwww}\text{\ \ \ \ \ (by definition of $U(\eps_{m})$ in Definition~\ref{def_detectable_abrupt})}\\
&\le \sum_{m=0}^M \Crabrupt \sqrt{16\wasD  (\log (T^3)) (T_{m+1} - T_{m})}
\phantom{www}\text{\ \ \ \ \ (by $\min(a,b) \le \sqrt{ab}$)}\\
&= \Crabrupt \sqrt{16\wasD  (\log (T^3)) (M+1) T}\nn
&\phantom{wwwww} \text{\ \ \ (by the Cauchy-Schwarz inequality and $\sum_m (T_{m+1} - T_m) = T$)}\\
&= \tilO(\sqrt{MKT}).
\end{align}
\end{proofof}

\subsection{Regret in gradual environment: proof of Theorem \ref{thm:regret_gradual}}
\label{subsec:regret_gradual}

We first state Lemmas \ref{lem:gradualbound_bandit} and \ref{lem:errorshrink_bandit}, then go to the proof of Theorem \ref{thm:regret_gradual}. 
The high-level implications of these lemmas are as follows.
Lemma \ref{lem:gradualbound_bandit} limits the possible amount of drift such that no reset occurs.
Lemma \ref{lem:errorshrink_bandit} bounds the number of resets as $\Mbreak := |\ET_d| = \tilO(T^{1-2d/3})$.

\begin{lem}\label{lem:gradualbound_bandit}
With probability at least $1-2K/T$,
for any $N < |W(t)|, t \in [T]$, $i \in [K]$ and $s,s' \in W(t)$
\begin{equation}\label{ineq:gradualbound_bandit}
|\mu_{i,s} - \mu_{i,s'}| \le
\frac{1}{\Cdetect}
\left(3 b \wasD  N + \tilO\left( \sqrt{1/N} \right) \right)
\end{equation}
holds.
\end{lem}
Lemma \ref{lem:gradualbound_bandit} is a version of Lemma \ref{lem:gradualbound} for the bandit setting. This case is much more challenging mainly because the monitoring arm may can change among blocks.

\begin{proofof}{Proof of Lemma \ref{lem:gradualbound_bandit}}
We use a tuple $(l,c)$ to represent the $c$-th subblock of the $l$-th block for $l=1,2,\dots,$ and $c=1,2,\dots,2^{l-1}$, that is,
the window consisting of $(\wasD N(2^{l-1}+c-2)+1,\dots,\wasD N(2^{l-1}+c-1))$-th rounds after the last reset.
We write $\headt_l$ (resp.~$\tailt_l$) for the first (resp.~last) round of the $l$-th block, that is,
$\headt_l=\wasD N(2^{l-1}-1)+1$ and $\tailt_l=\wasD N(2^{l}-1)$.
The window $W_{:(l,c)}$ consists of all subblocks before
$W_{(l,c)}$ (not including $W_{(l,c)}$).
Similarly, subwindow $W_{(l,c):(l,c')}$ for $c<c'$ denotes the joint window consisting of
$W_{(l,c)}, W_{(l,c+1)},\dots,W_{(l,c'-1)}$.

Fix an arbitrary $l\in\mathbb{N}$. The second property of the monitoring consistency (Definition \ref{def:monitoring}) implies the following:
Assume that no reset occurred up to the $l$-th block. There exists an arm that is drawn at least $N$ times for each subblock $c = 1,2,\dots,2^{l-1}$ in the $l$-th block. Moreover, this arm is drawn at least $N$ times in the final subblock of the $(l-1)$-th block.

By the property above and the fact that no reset occurs up to subblock $(l,c)$, there exists $i_l$ such that
for any $l\in \mathbb{N}$ and $c\in [2^{l-1}]$
\begin{align}%
|\hatmu_{i_l, W_{(l,c)} }-\hatmu_{i_l, W_{:(l,c)} }|
&\le 2 \sqrt{\frac{\log (T^3)}{\myred{2N}}},\label{gap1}\\
|\hatmu_{i_l, W_{(l,1):(l,c)} } - \hatmu_{i_l, W_{:(l,1)} }|
&\le
2 \sqrt{\frac{\log (T^3)}{2N}}\label{gap2}
\end{align}
because otherwise a reset should occur.
Then, for any $l\ge 2$ and $2 \le c\le 2^{l-1}$ we have
\begin{align}
\lefteqn{
|\hatmu_{i_l, W_{(l,c)} }-\hatmu_{i_l, W_{(l,1)} }| 
}\nn
&\le
|\hatmu_{i_l, W_{:(l,c)} }-\hatmu_{i_l, W_{(l,1)} }| 
+2 \sqrt{\frac{\log (T^3)}{2N}}\since{by \eqref{gap1}}
\nn
&=
\left|\frac{\niw{i_l}{(l,1):(l,c)}\hatmu_{i_l, W_{(l,1):(l,c)}}+(\niw{i_l}{:(l,c)}-\niw{i_l}{(l,1):(l,c)})\hatmu_{i_l, W_{:(l,1)}}}{\niw{i_l}{:(l,\myred{c})}}-\hatmu_{i_l, W_{(l,1)} }\right| 
+2 \sqrt{\frac{\log (T^3)}{2N}}
\nn
&\le
\left|\frac{\niw{i_l}{(l,1):(l,c)}\hatmu_{i_l, W_{:(l,1)}}+(\niw{i_l}{:(l,c)}-\niw{i_l}{(l,1):(l,c)})\hatmu_{i_l, W_{:(l,1)}}}{\niw{i_l}{:(l,\myred{c})}}-\hatmu_{i_l, W_{(l,1)} }\right| 
+4 \sqrt{\frac{\log (T^3)}{2N}}\since{by \eqref{gap2}}
\nn
&=
\left|\hatmu_{i_l, W_{:(l,1)}}-\hatmu_{i_l, W_{(l,1)} }\right| 
+4 \sqrt{\frac{\log (T^3)}{2N}}
\nn
&\le
6 \sqrt{\frac{\log (T^3)}{2N}}\since{by \eqref{gap1}}.\label{mu_gap_inblock}
\end{align}
Also, for $c=1$, \eqref{mu_gap_inblock} is trivial.
For $l=1$, it is also trivial since $c=1$ must hold from $c\le 2^{l-1}$.
By following the same discussion, we also have
\begin{align}
|\hatmu_{i_l, W_{(l,c)} }-\hatmu_{i_l, W_{(l,2^{l-1})} }| 
&\le
6 \sqrt{\frac{\log (T^3)}{2N}}.\label{mu_gap_inblock2}
\end{align}

By Eq.~\eqref{ineq:allwindowsbound_bandit}
\updated{with $p=1$} 
we have
\begin{align}
|\mu_{i, W_{l,c} } - \hatmu_{i, W_{l,c} }| &\le \sqrt{\frac{\log (T^3)}{2N}}
\label{ineq:conf_eachregion_rev_bandit}
\end{align}
for all $l\in\mathbb{N}, c\in 2^{n-1}$
with probability at least $1 - 2K/T$.
By the fact that $\mu_t$ moves at most $b\wasD N$ within a subblock of size $\wasD N$,
\begin{align}
|\mu_{i, W_{(l,c)} } - \mu_{i, t}| &\le b\wasD N \qquad\mbox{for round $t$ in subblock $W_{l,c}$}\nn
|\mu_{i, W_{(l-1,2^{l-2})} } - \mu_{i, W_{(l,1)}}| &\le b\wasD N.
\label{ineq:drift_each_rev_bandit}
\end{align}

Now, let $W_{l,c}$ (resp. $W_{l',c'})$ be the subwindow that $s$-th (resp., $s'$-th) round belongs to.
Here, we assume without loss of generality that $s<s'$.
Then we have
\begin{align}
\lefteqn{
|\mu_{i,s} - \mu_{i,s'}|
}\nn
&\le
|\mu_{i,s}-\mu_{i,\bar{t}_l}|
+
|\mu_{i,\tailt_l} - \mu_{i,\headt_{l+1}}|
+
|\mu_{i,\headt_{l+1}} - \mu_{i,s'}|
\nn
&\le
|\mu_{i,s}-\mu_{i,\bar{t}_l}|
+
b
+
|\mu_{i,\headt_{l+1}} - \mu_{i,s'}|
\label{trans_from}\\
&\le
\frac{1}{\Cdetect}|\mu_{i_l,s}-\mu_{i_l,\bar{t}_l}|
+
b
+
|\mu_{i,\headt_{l+1}} - \mu_{i,s'}|
\nn
&\le
\frac{1}{\Cdetect}
(
|\mu_{i_l,W_{(l,c)}}-\mu_{i_l,W_{(l,2^{l-1})}}|
+2b\wasD N)
+
b
+
|\mu_{i,\headt_{l+1}} - \mu_{i,s'}|
\since{by \eqref{ineq:drift_each_rev_bandit}}
\nn
&\le
\frac{1}{\Cdetect}
\left(
8
\sqrt{\frac{\log (T^3)}{2N}}
+2b\wasD N
\right)
+
b
+
|\mu_{i,\headt_{l+1}} - \mu_{i,s'}|,
\since{by \eqref{mu_gap_inblock2} and \eqref{ineq:conf_eachregion_rev_bandit}}
\label{trans_to}
\end{align}
and recursively applying this transformation for $l,l+1,l+2,\dots,l'-1$,
we have
\begin{align}
\lefteqn{
|\mu_{i,s} - \mu_{i,s'}|
}\nn
&\le
\frac{l'-l}{\Cdetect}
\left(
8
\sqrt{\frac{\log (T^3)}{2N}}
+2b\wasD N
\right)
+
b(l'-l)
+
|\mu_{i,\headt_{l'}} - \mu_{i,s'}|
\nn
&\le
\frac{l'-l+1}{\Cdetect}
\left(
8
\sqrt{\frac{\log (T^3)}{2N}}
+2b\wasD N
\right)
+
b(l'-l),\n
\end{align}
where we obtain the last inequality by applying the same transformation as
that from \eqref{trans_from} to \eqref{trans_to}.
We obtain the desired result since $l'\le \log_2 T=O(\log T) = \tilO(1)$.
\end{proofof}

\begin{lem} \label{lem:errorshrink_bandit}
There exists an event $\EX$ that holds with probability at least
\begin{equation}
\Pr[\EX] \le \tilde{O}(K T^{-1})
\end{equation}
such that, under $\EX^c$,  
\begin{equation}\label{ineq_mdsize}
M_d \le T/(C_1 b^{-2/3})
=\tilde{O}(T^{1-2d/3})
\end{equation}
holds.
\end{lem}
Lemma \ref{lem:errorshrink_bandit} is a version of Lemma \ref{lem:errorshrink} for the bandit setting. 
In the following, we derive Lemma \ref{lem:errorshrink_bandit} for completeness. The steps are very similar to the proof of Lemma \ref{lem:errorshrink}.
\begin{proofof}{Proof of Lemma \ref{lem:errorshrink_bandit}}
Let  
\begin{multline}
\EX_j(t) = \bigcup_{W_1, W_2 : W(t) = W_1 \cup W_2, j \in [K]} 
\left\{ |W_1| \le C_1 b^{-2/3}, |W_2| \le C_1 b^{-2/3}, |\hatmu_{j, W_1} - \hatmu_{j, W_2}| \ge  
\epscut^\delta 
\right\}
\end{multline}
where $C_1 = \tilO(1)$ is a factor 
that is specified in
Eq.~\eqref{ineq:defconebandit}.
Let 
$
\EX = \bigcup_{t \in [T], j\in[K]} \EX_j(t).
$
Let 
\begin{equation}
\EW_{C_1} = \{W_0 \in \EW: |W_0| \le C_1 b^{-2/3}\}
\end{equation}
be the set of windows of size at most $C_1 b^{-2/3}$. It is easy to show that $|\EW_{C_1}| \le T C_1 b^{-2/3}$.
For a fixed window $W$ and $i \in [K]$, Hoeffding inequality states that 
\begin{equation}\label{ineq_hoeffding}
|\mu_{i,W} - \hatmu_{i,W}| >
\sqrt{\frac{\log (\eta^{-1})}{2|W_i|}},
\end{equation}
occurs with probability at most $2\eta$.
By the union bound of Eq.~\eqref{ineq_hoeffding} with $\eta^{-1} = T^{2+d}$ over all $i$ and all windows of $\EW_{C_1}$ and all $K$ arms, with probability at least 
\begin{equation}
    1 - \frac{2K}{T^{2+d}} \times T C_1 b^{-2/3} = 1 - 2 C_1 K T^{-1} b^{1/3} \ge 1 - 2 C_1 K T^{-1},
\text{\ \ \ \ (by $T^{-d} = b$)}
\end{equation}
\updated{
In summary,
\begin{equation}\label{ineq_xbase}
\bigcap_{i, W\in \EW_{C_1}}
|\mu_{i,W} - \hatmu_{i,W}| \le
\sqrt{\frac{\log (T^{2+d})}{2|W_i|}}
\end{equation}
holds with probability at least $1 - 2 C_1 K T^{-1}$. 
In the following, we show that $\EX$ never occurs under the event of Eq.~\eqref{ineq_xbase}. In particular we use the proof-by-contradiction argument on the split event $|\hatmu_{j, W_1} - \hatmu_{j, W_2}| \ge \epscut^\delta $.
}
Eq.~\eqref{ineq_xbase} implies
\begin{align}
\label{ineq:gradual_oneplusc_bandit}
|\mu_{j,W_1} - \hatmu_{j,W_1}| &\le \sqrt{\frac{\log ( T^{2+d})}{2|W_{j,1}|}} \nn
|\mu_{j,W_2} - \hatmu_{j,W_2}| &\le \sqrt{\frac{\log ( T^{2+d})}{2|W_{j,2}|}}
\end{align}
for all $j, t$ and any split $W_1 \cup W_2 = W(t): |W_1|, |W_2| \le C_1 b^{-2/3}$. Let $N = C_1 b^{-2/3}$.
By the definition of the gradually changing stream, 
\begin{equation}\label{ineq:gradual_driftwidth_bandit}
|\mu_{j,W_1} - \mu_{j,W_2}|
\le
2bN.
\end{equation}

\updated{Assuming that $|\hatmu_{j, W_1} - \hatmu_{j, W_2}| \ge \epscut^\delta$}, it holds that
\begin{align}
\epscut^\delta 
&= \sqrt{ \frac{\log (T^3)}{2|W_{j,1}|} } + \sqrt{ \frac{\log (T^3)}{2|W_{j,2}|} }
\text{\ \ \ (by the definition of $\epscut^\delta$ in \eqref{def_epscut})}
\label{ineq_xcomp_three}\\ 
& \le |\hatmu_{j,W_1} - \hatmu_{j,W_2}|\\
& \le |\mu_{j,W_1} - \mu_{j,W_2}| + |\hatmu_{j,W_1} - \mu_{j,W_1}| + |\hatmu_{j,W_2} - \mu_{j,W_2}| \intertext{\ \ \ \ \ (by triangular inequality)} 
& \le 2bN + |\hatmu_{j,W_1} - \mu_{j,W_1}| + |\hatmu_{j,W_2} - \mu_{j,W_2}| 
\text{\ \ \ (by \eqref{ineq:gradual_driftwidth_bandit})} \\
& \le 2bN + \sqrt{\frac{\log (T^{2+d})}{2|W_{j,1}|}} + \sqrt{\frac{\log (T^{2+d})}{2|W_{j,2}|}}. 
\label{ineq_xcomp_last}
\text{\ \ \ (by \eqref{ineq:gradual_oneplusc_bandit})}
\end{align}
\updated{
Eq.~\eqref{ineq_xcomp_three} $\le$ Eq.~\eqref{ineq_xcomp_last} 
}
is equivalent to
\begin{equation}
\frac{\sqrt{3}-\sqrt{2+d}}{\sqrt{2}}
\left( 
\sqrt{\frac{\log T}{|W_{j,1}|}} +
\sqrt{\frac{\log T}{|W_{j,2}|}} 
\right)
\le 
2bN,
\end{equation}
which
cannot
hold for 
\begin{equation}\label{ineq:defconebandit}
N = b^{-2/3}\underbrace{\left( \frac{\sqrt{3}-\sqrt{2+d}}{2\sqrt{2}}
 \sqrt{\log T} \right)^{2/3}}_{=: C_1}, |W_{j,1}|, |W_{j,2}| \le N.
\end{equation}
Therefore, $\EX_j(t)$ never occurs under the event of  Eq.~\eqref{ineq_xbase}.
\end{proofof}

\begin{proofof}{Proof of Theorem \ref{thm:regret_gradual}}
Let 
$R(t)$ be the most recent reset before $t$, or $R(t)= 0$ if no reset has occurred yet.
Let 
\begin{equation}
\eps_R(t) := 
\max_{s: R(t)< s<t} \max_i |\mu_{i,s} - \mu_{i,R(t)+1}|,
\end{equation}
which is the amount of drift in view of the current \AADWIN{}-bandit.
By Eq.~\eqref{def_epst} and Lemma \ref{lem:gradualbound_bandit}, the event
\begin{equation} %
\EY = \bigcap_{t \in [T]}
\left\{
\eps_R(t) \le (\CGrad)^{-1} \left( 3 b \wasD  N + \tilO\left( \sqrt{1/N} \right) \right)
\right\}
\end{equation}
holds with probability at least $1-2K/T$.

We assume
$\EX^c\cap \EY$, 
because it holds with probability $1 - O(K/T)$ and thus the regret under $\EX \cup \EY^c$ is negligible.
\updated{Eq.~\eqref{ineq_mdsize} states that} event $\EX^c$ implies that $\Mbreak = O(T^{1-2d/3})$. 
Event $\EY$ with $N = \tilde{\Theta}((b\wasD )^{-2/3})$ implies that 
\begin{equation}\label{ineq_epsr_size}
\eps_R(t) = \tilO\left(
(b\wasD )^{1/3}
\right)
\end{equation}
for $t \in [T]$.

By Eq.~\eqref{ineq_mdsize} we see that $\Mbreak \le T/(C_1b^{-2/3})$ under $\EX^c$.
\updated{
Letting $\Td{0}=1$ and $\Td{m}=T$ for $m > \Mbreak$, 
}
we first have
\begin{align}
\Ex\left[
\Reg(T)
\right]
&\le
\sum_{m=1}^{T/(C_1b^{-2/3})}
\Ex\left[
\sum_{t=\Td{m}+1}^{\Td{m+1}} \reg(t)
\right]
+
\Ex\left[
\Ind[\EX^c]\Reg(T)
\right]
\\
&\le
\sum_{m=1}^{T/(C_1b^{-2/3})}
\Ex\left[
\sum_{t=\Td{m}+1}^{\Td{m+1}} \reg(t)
\right]
+
O(1).
\end{align}
Let $\Nduring{i}{m} = \sum_{t=\Td{m}+1}^{\Td{m+1}} \Ind[I(t)=i]$ be the number of draw on arm $i$ between the $m$-th and $(m+1)$-th reset, and 
$\Delta_{i,m} = \max_j \mu_{j,\Td{m}+1} - \mu_{i,\Td{m}+1}$ be the gap at the first round after the $m$-th reset. 
Let $\mathcal{H}_m$ be the history until the $m$-th reset.
Then the regret between the $m$-th and $(m+1)$-th reset is bounded by Lemma \ref{lem_monitoringdtprop} and we have
\begin{align} 
\lefteqn{
\Ex\left[\sum_{t=\Td{m}+1}^{\Td{m+1}}
\reg(t)\right]
}\nonumber\\
&=
\Ex\left[
\Ex\left[\left.
\sum_{t=\Td{m}+1}^{\Td{m+1}} \reg(t)
\right| \mathcal{H}_{m}\right]\right]
\\
&\le 
\Ex\left[
\sum_i
\left(
\min\left\{
\Delta_{i,m}\Ex[\Nduring{i}{m}|\mathcal{H}_m],
O\left(\frac{\log T}{\Delta_{i,m}}\right)
\right\}
+ 
\Ex\left[\left.
\Nduring{i}{m}
\max_t \eps_R(t)
\right|\mathcal{H}_m
\right]
\right)
\right].
\end{align}
Here, the summation over $m$ for the first term is transformed as follows.

\begin{align}
\lefteqn{
\sum_{m=1}^{T/(C_1b^{-2/3})}
\Ex\left[
\sum_i
\min\left\{
\Delta_{i,m}\Ex[\Nduring{i}{m}|\mathcal{H}_m],
O\left(\frac{\log T}{\Delta_{i,m}}\right)
\right\}
\right]
}\nonumber\\
&\le
\sum_{m=1}^{T/(C_1b^{-2/3})}
\Ex\left[
\sum_i
O\left(\sqrt{
\Ex[\Nduring{i}{m}|\mathcal{H}_m]
\log T}\right)
\right]
\text{\ \ \ \ (by $\min(a,b) \le \sqrt{ab}$)}\\
&\le
\sum_{m=1}^{T/(C_1b^{-2/3})}
\sum_i
O\left(\sqrt{
\Ex\left[
\Ex[\Nduring{i}{m}|\mathcal{H}_m]
\right]
\log T
}\right)
\\
&=
\sum_{m=1}^{T/(C_1b^{-2/3})}
\sum_i
O\left(\sqrt{
\Ex\left[
\Nduring{i}{m}
\right]
\log T
}\right)\\
&\le 
O\left(
\sqrt{
T
\left(
T/(C_1b^{-2/3})
\right)
K
\log T
}
\right) = \tilde{O}(Tb^{1/3}\sqrt{K}),
\end{align}
where, in the second and third inequalities, we have applied Jensen's inequality on 
$T/(C_1b^{-2/3}) \times K$
elements such that
$
\sum_{m=1}^{T/(C_1b^{-2/3})}
\sum_i
\Ex\left[
\Nduring{i}{m}
\right]
= T
$
holds.
For the summation over the second term we have
\begin{align}
\lefteqn{
\sum_{m=1}^{T/(C_1b^{-2/3})}
\Ex\left[\sum_i
\Ex\left[\left.
\Nduring{i}{m}
\max_t \eps_R(t)
\right|\mathcal{H}_m
\right]\right]
}\nonumber\\
&\le
\sum_{m=1}^{T/(C_1b^{-2/3})}
\Ex\left[\Ind[\EY]\sum_i
\Ex\left[\left.
\Nduring{i}{m}
\max_t \eps_R(t)
\right|\mathcal{H}_m
\right]\right]
+
\sum_{m=1}^{T/(C_1b^{-2/3})}
\Ex\left[\Ind[\EY^c]\sum_i
\Ex\left[\left.
\Nduring{i}{m}
\max_t \eps_R(t)
\right|\mathcal{H}_m
\right]\right]
\\
&\le
\sum_{m=1}^{T/(C_1b^{-2/3})}
\Ex\left[\sum_i
\Ex\left[\left.
\Nduring{i}{m}
\max_t \tilde{O}((bK)^{1/3})
\right|\mathcal{H}_m
\right]\right]
+
\sum_{m=1}^{T/(C_1b^{-2/3})}
\Ex\left[\Ind[\EY^c]\sum_i
\Ex\left[\left.
\Nduring{i}{m}
\right|\mathcal{H}_m
\right]\right]\nn
&\phantom{wwwwwwwwwwwwwwwwwwwwwwwwwwwwwwwwwww}\text{\ \ \ \ (by Eq.~\eqref{ineq_epsr_size})}\\
&=
\tilde{O}((bK)^{1/3})\sum_{m=1}^{T/(C_1b^{-2/3})}
\sum_i
\Ex\left[
\Nduring{i}{m}
\right]
+
\sum_{m=1}^{T/(C_1b^{-2/3})}
\Ex\left[\Ind[\EY^c]\sum_i
\Ex\left[\left.
\Nduring{i}{m}
\right|\mathcal{H}_m
\right]\right]\\
&=
\tilde{O}(T(bK)^{1/3})+O(1).
\end{align}
In summary,
\begin{align}
\Ex\left[
\Reg(T)
\right]
\le \tilde{O}(Tb^{1/3}\sqrt{K})+\tilde{O}(T(bK)^{1/3})+O(1) = \tilde{O}(Tb^{1/3}\sqrt{K}).
\end{align}

\end{proofof}

\end{document}